\documentclass[final,5p,times,twocolumn]{elsarticle}
\usepackage{amssymb}
\usepackage{amsmath,bm}
\usepackage{amsthm}
\usepackage{graphicx}
\usepackage{float}
\usepackage{subfig}
\usepackage{amsmath}
\usepackage{algpseudocode}
\usepackage{algorithmicx,algorithm}
\usepackage{array}
\usepackage{multirow}
\usepackage{color}
\usepackage{booktabs}
\usepackage{siunitx}
\usepackage{nicematrix}
\usepackage{tikz}
\newtheorem{theorem}{Theorem}
\newtheorem{manualtheoreminner}{Theorem}
\newenvironment{manualtheorem}[1]{%
  \renewcommand\themanualtheoreminner{#1}%
  \manualtheoreminner
}{\endmanualtheoreminner}
\journal{arXiv}
\newtheorem{example}{Example}%
\newtheorem{remark}{Remark}%
\newtheorem{lemma}{Lemma}%

\newtheorem{definition}{Definition}%

\begin{document}

\begin{frontmatter}
	\title{ABIGX: A Unified Framework for eXplainable Fault Detection and Classification}

	\author[address1]{Yue Zhuo}\ead{zhuoy1995@zju.edu.cn}    
	\author[address1]{Jinchuan Qian}\ead{qianjinchuan@zju.edu.cn}             
	\author[address2]{Zhihuan Song}\ead{zhsong@iipc.zju.edu.cn}
	\author[address1]{Zhiqiang Ge\corref{1}} 
	\cortext[1]{corresponding author. Tel: +86-571-87951442.}
	\ead{gezhiqiang@zju.edu.cn}  
	
	\address[address1]{ State key Laboratory of Industrial Control Technoogy, College of Control Science and Engineering, Zhejiang University, Hangzhou, 310027, PR China} 
	\address[address2]{ School of Automation, Guangdong University of Petrochemical Technology, Maoming, 525000, Guangdong, China}

	\begin{keyword}
		eXplainable Artificial Intelligence (XAI), Fault Detection, Fault Classification, Fault Diagnosis, Variable Contribution
	\end{keyword}
	
	\begin{abstract}
	For explainable fault detection and classification (FDC), this paper proposes a unified framework, ABIGX (Adversarial fault reconstruction-Based Integrated Gradient eXplanation). ABIGX is derived from the essentials of previous successful fault diagnosis methods, contribution plots (CP) and reconstruction-based contribution (RBC). It is the first explanation framework that provides variable contributions for the general FDC models. The core part of ABIGX is the adversarial fault reconstruction (AFR) method, which rethinks the FR from the perspective of adversarial attack and generalizes to fault classification models with a new fault index. For fault classification, we put forward a new problem of fault class smearing, which intrinsically hinders the correct explanation. We prove that ABIGX effectively mitigates this problem and outperforms the existing gradient-based explanation methods. For fault detection, we theoretically bridge ABIGX with conventional fault diagnosis methods by proving that CP and RBC are the linear specifications of ABIGX. The experiments evaluate the explanations of FDC by quantitative metrics and intuitive illustrations, the results of which show the general superiority of ABIGX to other advanced explanation methods.
	\end{abstract}
	
\end{frontmatter}

\section{Introduction}
\par Data-driven fault detection and classification (FDC) is an important technique for recognizing patterns of faults in industrial processes. Fault detection~\cite{5282515} decides whether a fault has occurred, which is realized by measuring the residual projection of unsupervised models like Principal Component Analysis (PCA)~\cite{9430765} and AutoEncoder (AE)~\cite{9632460}. Fault classification is to determine the types of faults, where the artificial neural networks (NN)~\cite{yadav2014overview} have achieved great success recently.
\par Nevertheless, these complex models become increasingly opaque, tending to be black boxes for the users. This opaque leads to the lack of trust in the model output (prediction), which is vital for practical applications in critical fields (e.g., industrial processes, autonomous systems and healthcare). The eXplainable AI (XAI)~\cite{surveyXAI} is to increase the transparency and trust of AI models, by seeking explanations for how predictions come about.

\par As a subfield of XAI, the eXplainable FDC (XFDC) targets to help people understand the data-driven FDC models and keep the human in loops of industrial processes. Given a fault sample, XFDC indicates how much each variable contributes to the monitoring results (detection score or classification confidence). The variables with high contributions are considered as the root cause of the anomaly (e.g., faulty sensors). The variable contributions by XFDC can support users in explaining the faults predicted by models and proactively repairing corresponding defective components.


\par Most previous works only focus on explaining fault detection models, known as the task of fault diagnosis. There are two most canonical fault diagnosis algorithms, contribution plots (CP)~\cite{miller1993contribution} and reconstruction-based contribution (RBC)~\cite{alcala2009reconstruction} based on fault reconstruction (FR)~\cite{FR}, aiming at the linear PCA-based detectors. Recently, many variants~\cite{hallgrimsson2020improved,qian2020locally,tan2019multi,alcala2010reconstruction,deng2020sparse} extended them to nonlinear detection models (e.g., AE) with novel fault indices. 

\begin{figure}
	\centering
\includegraphics[width=0.48\textwidth]{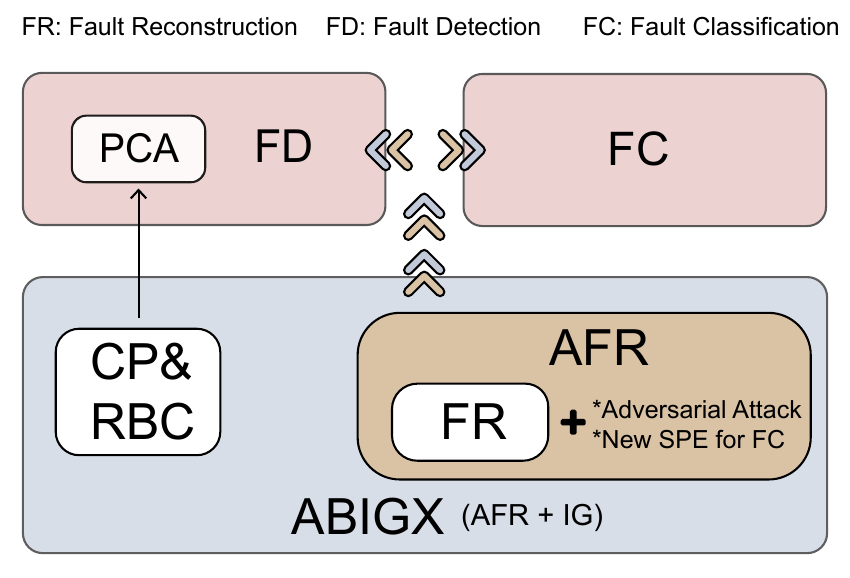}
\caption{The illustration of ABIGX and conventional fault diagnosis methods (RBC, CP, and FR). AFR generalizes the FR with adversarial attacks and a new SPE index for FC; ABIGX reframes linear RBC and CP to a unified framework combining AFR with Integrated Gradients (IG); ABIGX is applicable for the explanations of general FDC models.}
\label{fig_concepts}
\end{figure}

\par Distinctively, we first rethink FR from the adversarial attack perspective~\cite{goodfellow2014explaining} and prove FR shares the same optimization problem with the adversarial attack. We propose AFR, utilizing adversarial attacks to reconstruct faults in general FDC models\footnote{We refer to \emph{general FDC models} as fault detection and classification with various model modalities, including but not limited to PCA, AE, Multilayer Perceptron (MLP), and Convolutional Neural Network (CNN).}. Besides, we propose SPE fault index of the classification version, which enables AFR to achieve better fault reconstruction than vanilla attack algorithms.

\par ABIGX is based on AFR, which calculates the variable contributions by integrating gradients along the path from the explained samples to AFR-reconstructed samples. For explainable fault classification, we raise the fault class smearing problem, which is the intrinsic effect causing the incorrect variable contributions. Then we analyze the fault class smearing in the explainers of saliency map~\cite{simonyan2013deep}, Integrated Gradient (IG)~\cite{sundararajan2017axiomatic} and ABIGX, among which we prove that ABIGX performs best in mitigating this effect. For explainable fault detection, we bridge ABIGX with the canonical fault diagnosis methods by proving that CP and RBC are the linear specifications of ABIGX. 



\par Fig \ref{fig_concepts} illustrates the relationships between the major methods involved in the paper. In summary, the contributions of this work are four-fold:
\begin{itemize}
	\item We rethink the fault reconstruction from the adversarial attack perspective and propose AFR for the general FDC models. The new SPE index enables AFR to achieve reliable fault reconstruction for classification, which was less considered before. (Section \ref{sec_afr}) 
	\item We propose a unified framework, ABIGX, for effectively explaining general FDC models. We demonstrate that ABIGX includes the conventional fault diagnosis methods in a bigger framework, by proving CP and RBC are the linear specifications of ABIGX. (Section \ref{sec_abigx}) 
	\item We put forward a new problem in explainable fault classification, fault class smearing, which intrinsically causes inaccurate variable contributions. We analyze the performances of ABIGX, saliency map and IG under the fault class smearing and prove that our proposal achieves the best variable contributions that mitigate the smearing problem. (Section \ref{sec_FD}) 
	\item We perform experiments on both sensor (Tennessee Eastman Process) and image (wafer map fault) datasets with quantitative metrics for evaluating the XFDC. The results with deep insight show that ABIGX generally outperforms previous XAI methods. (Section \ref{sec_expr}) 
\end{itemize}
\par The remainder of the paper is organized as: Section \ref{sec_prelim} introduces the preliminaries and related works in the fields of FDC, fault diagnosis, adversarial attack, and XAI; Section \ref{sec_afr} proposes AFR by rethinking fault reconstruction from adversarial attack; Section \ref{sec_abigx} presents the unified framework, ABIGX, and proves its explainability by analyzing the fault class smearing problem; Section \ref{sec_FD} theoretically bridges the canonical fault diagnosis methods with ABIGX; Section \ref{sec_expr} reports the experiments for evaluating our proposal; Section \ref{sec_conclu} concludes the paper.

\section{Preliminaries and Related Works}
\label{sec_prelim}
\subsection{FDC}

\par \textbf{Fault detection} is to model the normal (working condition) data by statistical models. The models divide samples into a principal component subspace and a residual subspace, respectively containing normal signals and noise. Then, the \emph{control limit} is defined by the statistical indices in these two subspaces. Given an observed sample, the fault is detected if its index goes beyond the control limit. Early fault detection models are based on PCA~\cite{nomikos1995multivariate,chiang2000fault,joe2003statistical}, and recent deep learning methods mainly use the AE as the basic model~\cite{8386786,9199886,abid2021review}.
\par Given a dataset $\mathbf{X}=\{\mathbf{x}_1, \mathbf{x}_1, \cdots, \mathbf{x}_m\}\in \mathbb{R}^{m\times n} $, in which the sample of $n$ variables, $\mathbf{x}\in \mathbb{R}^{n}$, is free from faults. The fault detection models are trained on $\mathbf{X}$ to minimize the residual error. We define them as $f_{FD}(\mathbf{x})=\hat{\mathbf{x}}$ with the output of principal component samples\footnote{The output of AE model is commonly named as \emph{reconstructed samples} or \emph{reconstruction}. To avoid confusion with the fault reconstruction, this paper refers to the output sample of both PCA and AE as \emph{principal component sample} (i.e., the projection in principal component subspace).} $\hat{\mathbf{x}}\in \mathbb{R}^{m\times n} $.
\par \emph{PCA} model is defined as:
\begin{align}
	\label{eq_pcafd}
 f^{PCA}(\mathbf{x}) = \mathbf{P}\mathbf{P}^T\mathbf{x}
\end{align}
where $ \mathbf{P}\in \mathbb{R}^{n\times l}$ is the principal loading for the projection to principal component subspace $\hat{\mathbf{x}} =  \mathbf{P} \mathbf{P}^T\mathbf{x}$ and $l$ is the number of principal components. The loading matrix can be computed via eigendecomposition for the covariance matrix of $\mathbf{X}$.
\par \emph{AE} consists of two neural networks, an encoder and a decoder: 
\begin{align}
 f^{AE}(\mathbf{x}) = f_{dec}(f_{enc}(\mathbf{x}))
\end{align}
where the parameters of two networks are jointly trained by the back-propagation algorithms.
\par \emph{SPE index} for fault detection in this work is based on squared prediction error (SPE), which is the $\ell_2$ norm measurement of residual error:
\begin{align}
	\label{eq_SPE}
	SPE(\mathbf{x}) = \Vert \mathbf{x} - f_{FD}(\mathbf{x}) \Vert_2^2
\end{align}
\par The control limit $\delta^2$ is derived by the confidence limits of SPE values over normal data distribution. The sample is predicted as normal when its SPE is lower than the control limit. The fault detection models for binary label prediction $\hat{f}_{FD}:\mathbb{R}^{n}\rightarrow \{0,1\}$ is defined as:
\begin{align}
	\hat{f}_{FD}(x,\delta^2)&=\left\{
		\begin{aligned}
		&0,\ SPE(\mathbf{x})  \leq \delta^2 \\
		&1,\ otherwise
		\end{aligned}
		\right.
\end{align}

\par \textbf{Fault classification} is to recognize the different patterns in the fault samples. Unlike fault diagnosis, classification is a supervised task with labeled training datasets. This work focuses on the most popular classification model: neural networks.
\par Given a dataset $\mathbf{X}=\{\mathbf{x}_1, \mathbf{x}_1, \cdots, \mathbf{x}_m\}\in \mathbb{R}^{m\times n} $ and label $Y \in \{0,1,\cdots,k\}$ with $k$ being the number of fault types. The fault classifier $f_{FC}:\mathbb{R}^{n}\rightarrow \{0,1,\cdots,k\}$ is trained by:
\begin{align}
	f_{FC} = \arg \min_{f} \mathbb{E}_{\mathbf{x}\in \mathbf{X},\ y\in Y} J(y,f(\mathbf{x}))
\end{align}
where $J$ is the cross-entropy loss function for NN.

\subsection{Fault Diagnosis}
\par The explanations of fault detection models have been carefully studied before, which is referred to as fault diagnosis in the previous works~\cite{alcala2009reconstruction,miller1993contribution,qian2020locally,dong2019quality}. After a fault has been detected, the fault diagnosis methods are to diagnose the root cause of a fault by determining the contribution of each variable to the fault detection indices (e.g., SPE). We briefly discuss three canonical fault diagnosis methods that explain PCA models.

\par \textbf{Fault reconstruction}~\cite{FR} is to determine the necessary adjustment to bring the fault indices (e.g., SPE) back to normality. This necessary adjustment involves two parameters: direction $\mathbf{\xi}_i$ and magnitude $f_i$, where $f_i$ is a scalar and $\mathbf{\xi}_i$ is the $i$-th column of identity matrix indicating the $i$-th variable. For the sample containing five sensor variables, the direction $\mathbf{\xi}_i$ of $1$-st variable is:
\begin{align}
	\mathbf{\xi}_i = \left(\begin{array}{ccccc}
		1 &
		0 &
		0 &
		0 &
		0		
	\end{array} \right)^T
\end{align}
\par Most previous fault reconstruction methods predefine the reconstruction variable $\mathbf{\xi}_i$ in turn and solve the magnitude $f_i$. Given $\mathbf{\xi}_i$, the fault reconstruction is to find $f_i$:
\begin{align}
	f_i = \arg \min SPE(\mathbf{x}-f_i\mathbf{\xi}_i) 
	\label{eq_faultrecon}
\end{align}
where $\mathbf{x}$ is the fault sample and $\mathbf{x}-f_i\mathbf{\xi}_i$ is the reconstructed sample brought back to the normality.

\par The fault direction can be determined by the direction $\mathbf{\xi}_i$ such that the SPE of the reconstructed sample is minimized. However, fault reconstruction cannot provide accurate contributions for each variable, which is solved by RBC.

\par \textbf{RBC}~\cite{alcala2009reconstruction} is to calculate the fault index of the reconstruction vector $f_i\mathbf{\xi}_i$, which is expressed as:
\begin{align}
	\label{eq_rbcdef}
	RBC_i \equiv SPE(f_i\mathbf{\xi}_i) &= \Vert f_i\mathbf{\xi}_i - \mathbf{P}\mathbf{P}^T f_i\mathbf{\xi}_i \Vert_2^2 \notag \\
	&= (f_i\mathbf{\xi}_i)^T (\mathbf{I}-\mathbf{P}\mathbf{P}^T)f_i\mathbf{\xi}_i\notag \\
	&= (f_i\mathbf{\xi}_i)^T \tilde{\mathbf{C}} f_i\mathbf{\xi}_i
\end{align}
where $\tilde{\mathbf{C}}=\mathbf{I}-\mathbf{P}\mathbf{P}^T$ is the projection matrix to residual subspace.

\par \textbf{CP}~\cite{miller1993contribution} is directly constructed by determining the contribution of each variable to the detection index. Firstly, SPE defined can be expanded:
\begin{align}
	SPE(\mathbf{x}) =  \Sigma_{i=1}^{n}(\mathbf{x}_i - f_{FD}(\mathbf{x})_i)^2
\end{align}
where $\mathbf{x}_i$ and $f_{FD}(\mathbf{x})_i$ are the $i$-th variable of detected sample and principal component sample. The SPE contribution $C_i$ can be simply calculated by the squared value difference between the $i$-th variable of the input and output sample:
\begin{align}
	\label{eq_contri}
	C_i &= (\mathbf{x}_i - f_{FD}(\mathbf{x})_i)^2 \notag \\
		&= 	(\mathbf{\xi}_i^T\tilde{\mathbf{C}}\mathbf{x})^2  \qquad \text{for PCA}
\end{align}

\par \textbf{Our Distinctions:} Recently, fault diagnosis is dedicated to adapting RBC~\cite{qian2020locally,alcala2010reconstruction,ren2018new} and CP~\cite{hallgrimsson2020improved,tan2019multi} to nonlinear models. However, most of them are built on the conventional framework of fault reconstruction, which considered the variable contribution independently, while the interconnection between variables is less considered, which could be a significant problem for the nonlinear models. The distinctions of our proposal are mainly two-fold: 1) We consider fault reconstruction from the adversarial attack perspective, the AFR of which could automatically solve the variable combinations that contribute most to the fault index. 2) For the first time, we unify CP and RBC into a new framework, ABIGX.

\subsection{Gradient-based XAI methods}
\par Gradient-based XAI method can directly attribute the input feature through backward propagation, which are computationally efficient and applicable for multiple model modalities. Recently, Integrated Gradient (IG)~\cite{sundararajan2017axiomatic} is a popular gradient-based method and shares some identical intrinsic properties with fault diagnosis methods. Hence, this paper mainly focuses on IG and another canonical method: saliency map (vanilla gradient)~\cite{simonyan2013deep}.

\par \textbf{Saliency map} directly provides the explanation by computing the gradient of output prediction w.r.t. the input variables, which is a natural analogue of model behavior:
\begin{align}
	Grad_i = \frac{\partial f(\mathbf{x})}{\partial x_i}
	\label{eq_grad}
\end{align}
where the $f$ is the SPE index for fault detection and logit value of the ground-truth class for fault classification.

\par \textbf{IG} is a popular XAI algorithm. Instead of simply computing the gradient, IG accumulates the gradients along a path $\gamma(\alpha)$ ($\alpha\in [0,1]$) from the explained sample $\mathbf{x}$ to a counterfactual baseline $\mathbf{x}'$\footnote{The counterfactual samples contrastively explain the FDC models by answering:``Which features cause the model to predict fault rather than normal?'' The normal sample is counterfactual to the explained anomaly sample.}. The IG contribution of $i$-th variable is defined as:
\begin{align}
	\label{eq_IG}
	IG_i = \int_{\alpha=0}^1 \frac{\partial f(\gamma(\alpha))}{\partial \gamma_i(\alpha)} \frac{\partial \gamma_i(\alpha)}{\partial \alpha}\,d\alpha
\end{align}
where $\gamma(0)=\mathbf{x}'$ is the baseline, $\gamma(1)=\mathbf{x}$ is the explained fault sample, and the $f$ is the same FDC output function as defined in the vanilla gradient.

\par \textbf{Our distinctions:} From the perspective of IG methods, ABIGX essentially proposes a novel baseline that is the reconstructed fault by AFR. Unlike previous IG-based methods that arbitrarily choose the baselines~\cite{sturmfels2020visualizing}, our AFR is exclusively designed for contrasting the fault sample and normality. AFR highlights the significant variables such that the fault samples are recognized by FDC models. The comparison between ABIGX and IG methods is also discussed in the following.

\subsection{Adversarial attack}
\par Adversarial attack is a critical security issue of ML models, adversarial robustness of FDC models has been studied in the work~\cite{9852307}. Concerning the attacks on fault samples, the adversarial examples $\mathbf{x}'$ is solved by:

\begin{align}
	\mathbf{x}'=\arg &\min f(\mathbf{x}')\notag \\
	s.t.\ &D(\mathbf{x},\mathbf{x}')<\eta, \notag\\
	& f(\mathbf{x}')< y
\label{eq_fdadv}
\end{align}
where $f$ is the same as defined in Eq. \ref{eq_grad}, $D$ is the distance measurement, $\eta$ is a predefined distance constraint, and $y$ is the fault label (the second constraint is not strictly necessary).

\par The adversarial attack and model explanation (especially the counterfactual explanation) are two closely-related fields. The counterfactual examples can be regareded as the generalization of adversarial examples. This (dis)similarities have been theoretically and empirically studied in many works~\cite{CEAE_3,CEAE_4}. As for fault diagnosis, the target of fault reconstruction can be regarded as searching the counterfactual examples, which derives our adversarial fault reconstruction (AFR) algorithm. 




\section{Fault Reconstruction: Adversarial Perspective} 
\label{sec_afr}
\par First of all, we study the fault reconstruction from the adversarial attack perspective and state that the existing fault reconstruction solves the same optimization problem with the sparse adversarial attack.

\begin{theorem}
	\label{theo_rbc}
	For fault detection, the existing fault reconstruction shares the same optimization objective with the sparse adversarial attack on fault samples.
\end{theorem}
\begin{proof}
	According to Eq. \ref{eq_fdadv}, the sparse adversarial attack on fault detection can be defined with the $\ell_0$ norm constrain (not strictly require the attack success):
	\begin{align}
		\mathbf{x}'=&\arg \min SPE(\mathbf{x}')\notag \\
		&s.t.\ \Vert \mathbf{x}',\mathbf{x}\Vert_0 < \eta 
	\end{align}
where $\Vert\cdot \Vert_0$ measures the non-zero element number of the vector and $\eta\in \mathbb{N}^+$ controls the variable number to be attacked.
\par If $\eta=1$ for attacking only one variable and the attacked variable direction $\xi_i$ are given, the attack objective becomes computing the magnitude $f_i$ on this direction with $\mathbf{x}' = \mathbf{x} + f_i\xi_i$:
\begin{align}
	\label{eq_afr_onebyone}
	f_i =\arg \min SPE(\mathbf{x}-f_i\xi_i),\ where\ \Vert\xi_i \Vert_0 = 1
\end{align}
which is the same as fault reconstruction defined in Eq. \ref{eq_faultrecon}.
\end{proof}
\subsection{Adversarial Fault Reconstruction}
\par Theorem \ref{theo_rbc} inspires us to use the adversarial attack algorithms for more generalized fault reconstruction, where the one-by-one predefining reconstruction direction is unnecessary. Iteratively predefining direction and solving magnitude could be time-consuming and ignore the interaction between variables.

\par Hence, we propose Adversarial Fault Reconstruction (AFR) for FDC models. Instead of searching the optimal fault directions one-by-one with brute force, AFR can automatically solve the (combination of) variables minimizing the SPE indices by applying the magnitude constraint on the distance between explained samples $\mathbf{x}$ and reconstructed samples $\mathbf{x}'_{AFR}$.
\begin{definition}
	\label{def_afr}
	(AFR for FDC) AFR reconstructs the fault in FDC models with the following optimization problem:
	\begin{align}
		\label{eq_afr}
		\mathbf{x}'_{AFR} = &\arg \min_{\mathbf{x}'_{AFR}} SPE(\mathbf{x}'_{AFR}) \notag \\
		&s.t.\, \Vert \mathbf{x} - \mathbf{x}'_{AFR}\Vert_p \leq \eta
	\end{align}
	 $\eta$ is to constrain the distance between reconstructed and original fault samples, the measurement of which is $\ell_p$ norm ($p=0,1,2$).
\end{definition}

\par Both constrained $\ell_1$ and $\ell_2$ norm magnitude could enable the attack algorithm automatically to focus on the variables that contribute most to the SPE index. AFR is more general since the variable interaction is more complex in the nonlinear models, and one-by-one variable search is more computationally expensive in the high-dimensional samples like images. 

\par Notably, we declare that the conventional one-by-one variable fault reconstruction can be regarded as a specification of AFR, which is named as AFR-OneVar:

\begin{definition}
	(AFR-OneVar) AFR-OneVar computes a set of reconstructed samples $\{\mathbf{x}'^{(1)}_{AFR^{OV}},\mathbf{x}'^{(2)}_{AFR^{OV}},\cdots, \mathbf{x}'^{(n)}_{AFR^{OV}} \}$ respectively for $n$ variables, where $n$ optimization problems are involved. 
	\par Given $i$-th variable direction $\xi_i$, the $\mathbf{x}'^{(i)}_{AFR^{OV}}$ can be defined as Eq. \ref{eq_afr_onebyone}:
	\begin{align}
		\mathbf{x}'^{(i)}_{AFR^{OV}} = &\arg \min SPE(\mathbf{x}'^{(i)}) \notag \\
		&s.t.\ \mathbf{x}'^{(i)}_{AFR^{OV}} - \mathbf{x} = f_i\xi_i
	\end{align}
where $i\in\{0,1,\cdots,n\}$.
\end{definition}
\par This is similar to some previous fault diagnosis methods~\cite{ren2018new} that extend RBC to nonlinear AE models. Fig \ref{fig_abigx2d} intuitively illustrates the comparison of AFR and AFR-OneVar with a 2-D PCA detection model. 

\par Previously, the SPE in Eq. \ref{eq_afr} is only defined for fault detection. To make AFR also applicable for fault classification, we introduce a novel SPE index for the classification version.

\begin{figure}
	\centering
\includegraphics[width=0.48\textwidth]{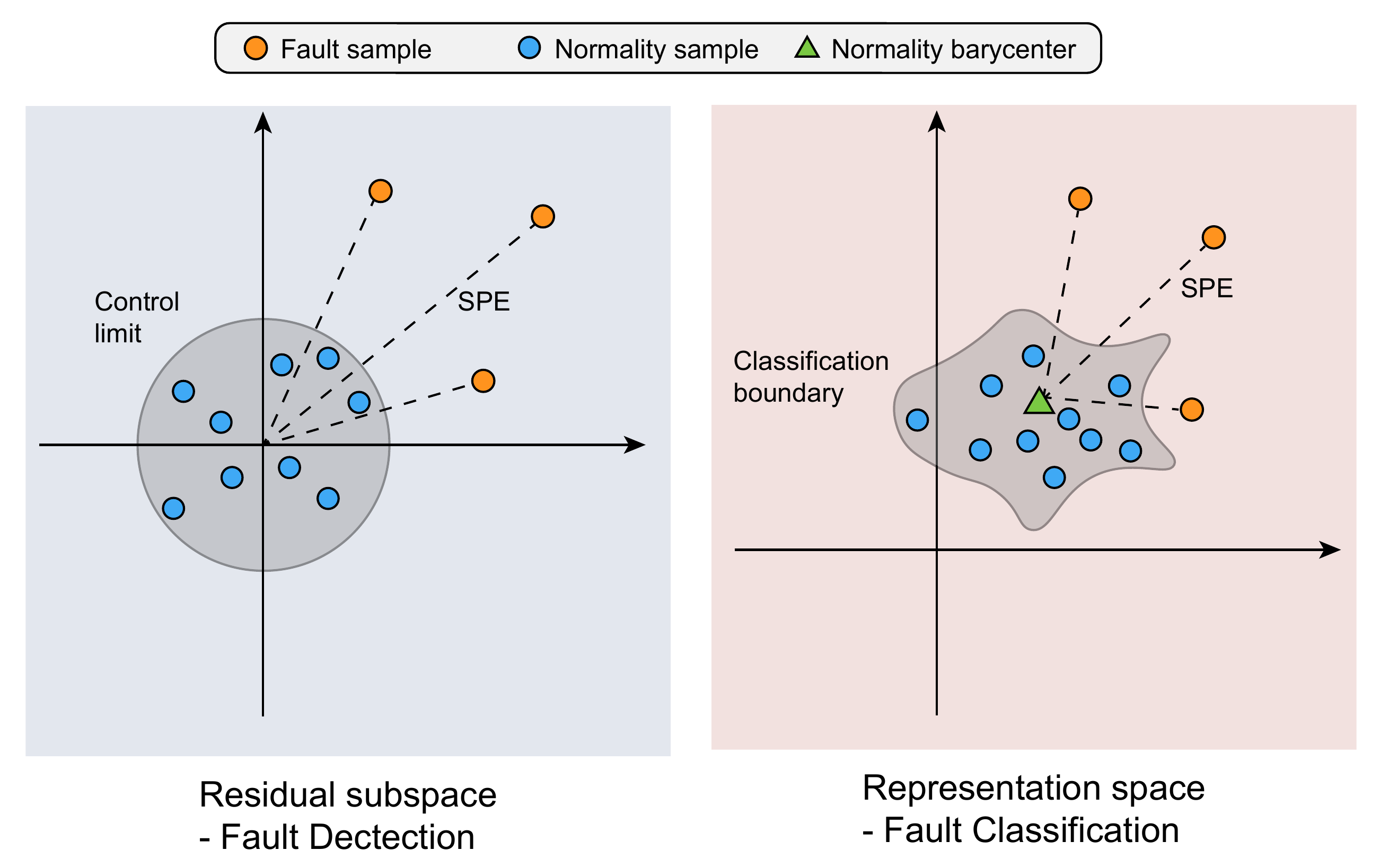}
\caption{SPE for fault classification (right), compared with the fault detection SPE (left).}
\label{fig_speFC}
\end{figure}
\subsection{SPE for Fault Classification}
\label{sec_spefc}
\par The SPE for fault classification is derived from fault detection (defined in Eq. \ref{eq_SPE}). Whereas the detection models are trained to minimize the residual projection of normal samples, we argue that the barycenter of normality samples' residual projection is on the zero point of residual subspace. Hence, the detection SPE for fault samples can be regarded as computing the Euclidean distance in residual subspace between fault samples' projection and barycenter of normality samples' projection (i.e., zero point). 
\par Thus, SPE of fault classification can be defined similarly. We compute the Euclidean distance in the representation space\footnote{Commonly, the penultimate layer's output of classification neural network is chosen as the representation.}. The barycenter of normality samples' representation is the ``zero point'' and the SPE is defined as the Euclidean distance between the normality barycenter and fault representation.
\par Fig. \ref{fig_speFC} illustrates the comparison with detection SPE and classification SPE. The definition of classification SPE is:

\begin{definition}
	(SPE for fault classification) Given the representation layer of classification network $h$, the set of normality samples $\mathbf{X}_{n}$, the classification SPE of fault sample $\mathbf{x}$ is:
	\begin{align}
		SPE_{FC}(\mathbf{x}) = \Vert h(\mathbf{x}) - \mathbb{E}_{\mathbf{x}_{n}\sim \mathbf{X}_{n} } [h(\mathbf{x}_{n})] \Vert_2^2
	\end{align}
	where the term $\mathbb{E}_{\mathbf{x}_{n}\sim\mathbf{X}_{n}}[h(\mathbf{x}_{n})]$ is the barycenter of normality representations.
\end{definition}
\par With defined SPE indices, AFR for FDC in Eq. \ref{eq_afr} can be solved with adversarial attack algorithms, which are briefly discussed in the following section.
\subsection{AFR Algorithms}
\par For different FDC models, the objective of AFR (Eq. \ref{eq_afr}) can be solved by different optimization algorithms. We mainly consider two specific algorithms: the mathematical programming that can guarantee the optimal solution and the gradient descent that is more general for multiple models.

\subsubsection{AFR via Mathematical Programming}
\par For the linear models, we can convert AFR into a convex mixed-integer linear programming (MILP), where the global minima can be solved. Here we specify the problem of $\ell_0$ norm constraint, which is appealing since it can directly diagnose the most $\eta$ significant variables in the FDC models. 
\par Firstly, a set of auxiliary binary variables $\rho$ is introduced to formulate the $\ell_0$ norm constraint into the mixed-integer linear form:
\begin{align}
	\label{eq_norm0}
	\Vert f\xi \Vert_0 < \eta  \equiv & \bigg[  (\sum_j\rho_j <\eta) \wedge (\mathbf{x}_j'\leq \mathbf{x}_j + \rho_j \cdot \epsilon) \wedge   \notag \\
	&  (\mathbf{x}_j' \geq \mathbf{x}_j - \rho_j \cdot \epsilon) \wedge (\rho_j \in \{0,1\})  \bigg]
  \end{align}
where conjunction symbol $\wedge$ denotes logical AND, $\mathbf{x}'_j$ denotes the $j$-th variable of the reconstructed sample, and $\epsilon$ is a large predefined bound for all variables.

\par Next, the objective $SPE(\mathbf{x}')$ of linear models is formulated by:
\begin{align}
	\mathbf{x}'^T\tilde{C}\mathbf{x}'
\end{align}
where $\tilde{C}$ is the residual projection matrix. 
\par This idea to solve AFR by MILP is similar to adversarial verification works~\cite{lomuscio2017approach,tjeng2018evaluating,cheng2017maximum}, which can be extended to ReLU-based NN. However, without loss of generality, this work applies the gradient-based attack algorithms for solving AFR, which is discussed in the following.


\subsubsection{AFR via Adversarial Attack}
\par For more general FDC models, we can realize AFR by utilizing the existing attack algorithms. In particular, we apply a multi-step gradient attack, Projected Gradient Descent (PGD)~\cite{PGD}, which is applicable for both $\ell_0$, $\ell_1$ and $\ell_2$ norm:
\begin{align}
	\mathbf{x}'_{t+1}=\mathbf{x}'_t - \frac{\nabla_{\mathbf{x}_t} SPE(\mathbf{x}_t) }{L}
	\label{eq_pgd}
\end{align} 
 where $\mathbf{x}'_{t}$ is the adversarial sample at $t$-th iteration and $L$ is to normalize the perturbation under certain norm.\footnote{If not specified, $AFR$ and $ABIGX$ use the $\ell_2$ norm as the default distance constraint.} 

\section{ABIGX: A Unified Framework}
\label{sec_abigx}
\par Aided by adversarial attack algorithms and the new classification SPE, AFR generalizes the fault reconstruction to general FDC models. Based on the counterfactual samples reconstructed by AFR, we are dedicated to computing the variables' contributions to the model prediction. 
\par Recalling the idea of RBC, RBC is to calculate the amount of model output (fault index) along the reconstruction vector $f_i\xi_i$. This implies the same motivation as the IG~\cite{IG} methods, which also computes the contribution amount to the model output along a path. The concrete relationship between RBC and IG will be rigorously proved in Section \ref{sec_FD}. 

\par Beyond the linear RBC and FR, IG and AFR correspondingly extend the variable contribution and fault reconstruction to general model modalities for both fault detection and classification tasks. Hence, by introducing the IG method into AFR, we propose ABIGX, which is a unified framework for explaining general FDC models. 

\subsection{Definition}
\par In general, ABIGX integrates gradients of model output along the path from the explained sample to the AFR-reconstructed sample.
\begin{definition}
	(ABIGX) ABIGX explains FDC models by attributing the variables: Given the explained sample $\mathbf{x}$ and reconstructed sample $\mathbf{x}'_{AFR}$ by AFR, the ABIGX contribution of $i$-th variable is defined as:
	\begin{align}
		\label{eq_ABIGX}
		ABIGX_i = &(\mathbf{x}_i - {\mathbf{x}'_{AFR}}_i) \times \notag \\
		& \int_{\alpha=0}^1 \frac{\partial f(\mathbf{x}'_{AFR} + \alpha(\mathbf{x} - \mathbf{x}'_{AFR}))}{\partial \mathbf{x}_i}\,d\alpha
	\end{align}
\end{definition}
\par Eq. \ref{eq_ABIGX} is a Riemann sum of Eq. \ref{eq_IG}, where the gradient is integrated along the straight-line path from $\mathbf{x}'_{AFR}$ to $\mathbf{x}$. If the reconstructed samples are solved in the one-by-one direction, the contribution can also be calculated along the direction one-by-one, which is named as ABIGX-OneVar:
\begin{definition}
	(ABIGX-OneVar) ABIGX-OneVar contribution of $i$-th variable is to integrate gradients along the path from $\mathbf{x}$ to the AFR-reconstructed sample $\mathbf{x}'^{(i)}_{AFR^{OV}}$ on variable direction $\xi_i$:

	\begin{align}
		\label{eq_ABIGX^{OV}}
		ABIGX^{OV}_i = &(\mathbf{x}_i - \mathbf{x}'^{(i)}_{AFR^{OV}}) \times \notag \\
		&   \int_{\alpha=0}^1 \frac{\partial f(\mathbf{x}'^{(i)}_{AFR^{OV}}+ \alpha(\mathbf{x} - \mathbf{x}'^{(i)}_{AFR^{OV}}))}{\partial \mathbf{x}_i}\,d\alpha  \notag \\
		 = &f_i\xi_i \times\int_{\alpha=0}^1 \frac{\partial f(\mathbf{x} + (1-\alpha)f_i\xi_i)}{\partial \mathbf{x}_i}\,d\alpha
	\end{align}
\end{definition}
\par ABIGX-OneVar omits the variables' interaction and can be regarded as the extension of RBC for general FDC models. Since both AFR and IG only require the gradient information of the model, ABIGX is a general explanation technique for FDC models.

\subsection{Fault class smearing and explainability analysis}
\par The saliency map and IG have been widely applied for the classifiers, but there is no theoretical analysis for the correctness of variable contributions for the fault classification. We put forward the \emph{fault class smearing} problem, which is derived from the recognition between different fault classes. We show that the saliency map directly suffers from fault class smearing and IG alleviates it, furthermore, ABIGX even outperforms IG and provides more accurate explanations with less fault class smearing.
\par Firstly, we construct a toy example for the fault classification. Without loss of generality, we assume different fault types happen on the different variables:
\begin{example}
	\label{examp_data}
	(Toy fault classification dataset) Let all the normal variables independently and identically sampled from a Gaussian distribution $\mathcal{N}(0,\sigma^2I)$, all faults have the same magnitude $f$ and corrupted variables subject to the distribution $\mathcal{N}(f,\sigma^2I)$. For simplicity, we assume the fault type $y$ has one corrupted variable on the direction $\xi_y$. Then the fault classification dataset with $M$ variables and $N$ fault types is defined by ($N<M$):
	 \begin{align}	
		\begin{cases}
			y = 0, \ & x_1,x_2,\cdots,x_M \overset{i.i.d}{\sim} \mathcal{N}(0,\sigma^2) \\
			y \sim \{1,2,\cdots,N\}, \ & x_y \sim \mathcal{N}(f,\sigma^2),\, \
				x_{j\neq y} \overset{i.i.d}{\sim} \mathcal{N}(0,\sigma^2) 
		\end{cases}
	 \end{align}
\end{example}
\par Secondly, we construct the linear classification model for the toy example to best separate all the classes:
\begin{definition}
	\label{def_linearclf}
	(Linear classifier) Let $\mathbf{W}=[\mathbf{w}_0,\mathbf{w}_1,\cdots,\mathbf{w}_N]^T$ be the weight of linear classifier, where its row vector $\mathbf{w}_y \in \mathbb{R}^{1\times M}$ ($y \in \{0,1,\cdots,N\}$) corresponds to the linear decision boundary of each fault (or normality) class. Based on Fisher's linear discriminant~\cite{liu2002gabor}, the optimal row vector is trained by maximizing the distance between class mean values and minimizing the variance within classes:
\begin{align}
	\label{eq_fld}
	\mathbf{w}_y^* = &\arg \max \frac{\mathbf{w}_y^T (\mathbf{\mu}_y - \overline{\mathbf{\mu}})(\mathbf{\mu}_y - \overline{\mathbf{\mu}})^T \mathbf{w}_y}{\mathbf{w}_y^T \Sigma \mathbf{w}_y} \notag \\
	& s.t.\ \Vert \mathbf{w}_y \Vert_1 = 1
\end{align}
where $\mathbf{\mu}_y$ is the mean vector of fault class $y$, $\overline{\mathbf{\mu}}$ is the mean of class means, and $\Sigma=\sum_{y=0}^{N}\sum_{\mathbf{x}\in\mathbf{x}_y}(\mathbf{x} - \overline{\mathbf{\mu}})(\mathbf{x}- \overline{\mathbf{\mu}})^T$ is the sum of covariances per class.
\end{definition}	
\par Based on the dataset in Example \ref{examp_data} and linear classifier in Definition \ref{def_linearclf}, we can prove the optimal model weights:
\begin{theorem}
	\label{theo_optweight}
		Let $w_{y,i}\in \mathbf{W}^*$ be the optimal model weight for variable $i$ on fault class $y$. For normality class $y=0$, the optimal $w_{0,i}$ is:
	\begin{align}
		w_{0,i} = \begin{cases}
			-\frac{1}{N}, & i \in \{1,\cdots,N\}\\
			0, & i \in \{N+1,\cdots,M\}
		\end{cases}
	\end{align}
	\par For fault class $y$, the optimal $w_{y,i}$ is: 
	\begin{align}
		w_{y,i} = \begin{cases}
			\frac{N}{2N-1},&\ i = y \\ 
			-\frac{1}{2N-1}, & i \in \{1,\cdots,N\}\setminus y \\
			0, & i \in \{N+1,\cdots,M\}
		\end{cases}
	\end{align}
\end{theorem}
\par The proof of Theorem \ref{theo_optweight} is stated in \ref{append_pr1}. Next, we demonstrate the fault class smearing effect in the existing gradient-based XAI methods. We state this effect is caused by the natural optimization objective of classifiers, which would lead to incorrect variable contributions and should be minimized during the explanation.

\begin{theorem}
	(Fault class smearing in saliency map) Given the explained fault type $y$ with the ground-truth variable $y$, the saliency map contributes are smeared into other irrelative variables:
	\begin{align}
		Grad_i =\begin{cases}
			\vert w_1 \vert, &\ i = y \\ 
			\vert w_2 \vert, & i \in \{1,\cdots,N\}\setminus y \\
		\end{cases}
	\end{align}
	\par We also define the degree of fault class smearing ($FCS$) by the ratio of the sum contribution of other irrelevant variables to that of variable $y$ contribution:
	\begin{align}
	FCS_{grad} = \frac{\sum_i^{i\neq y} Grad_i} {Grad_y} = \frac{N-1}{N}
	\end{align}
\end{theorem}

\par The intrinsic cause of fault class smearing is that the fault classifier not only distinguishes the fault sample and normality but also recognizes different fault types. This leads to the non-zero smearing weights on the irrelative variables that cause other fault types (e.g., $w_2$ weights), which can effectively distinguish different fault types. On the contrary, the variables that do not cause any faults (e.g., the least $M-N$ column weights in Eq. \ref{eq_matrix}) won't smear the weights.
\par These smearing weights on the irrelevant variables is inevitable in the fault classification models. A good explanation method should effectively mitigate the influence of smearing weights on the variable contributions. Next, we discuss fault class smearing in IG methods.
\begin{theorem}
	\label{theo_fcsig}
	(Fault class smearing in IG) Given the explained fault type $y$ with the ground-truth variable $y$ and the baseline of the expectation of normality samples $\mathbf{x}'=\overrightarrow{0} $, the expectation of IG contributes are:
	\begin{align}
		\label{eq_fcsig1}
		\mathbb{E}_{\mathbf{x}\sim \mathbf{X}_y} IG_i =\begin{cases}
			\vert fw_1 \vert, &\ i = y \\ 
			\vert\frac{\sigma\sqrt{2}}{\sqrt{\pi}} w_2 \vert, & i \in \{1,\cdots,N\}\setminus y \\
		\end{cases}
	\end{align}
	\par The degree of fault class smearing ($FCS$) in IG is:
	\begin{align}
		\label{eq_fcsig2}
	FCS_{IG} = \frac{N-1}{N} \frac{\sigma\sqrt{2} }{f\sqrt{\pi}} < FCS_{grad}
	\end{align}
\end{theorem}
\par The proof of Theorem \ref{theo_fcsig} is stated in \ref{append_pr2}. Theorem \ref{theo_fcsig} shows the IG mitigates the fault class smearing by introducing the input variable difference between the fault and normality. Since the fault variable always has a larger offset than other variables, its integrated gradient is also more significant.
\par Subsequently, we analyze fault class smearing in ABIGX, showing that ABIGX outperforms IG with a lower FCS degree.
\begin{theorem}
	\label{theo_abigx}
	(Fault class smearing in ABIGX) The fault class smearing degree in ABIGX is lower than that in IG:
	\begin{align}
		 FCS_{ABIGX}< FCS_{IG}  \notag
	\end{align}
\end{theorem}

\par The proof of Theorem \ref{theo_abigx} is stated in \ref{append_pr3}. Theorem \ref{theo_abigx} shows ABIGX effectively mitigates the fault class smearing problem by utilizing the gradient of classification SPE objective and outperforms IG which simply introduces the variable difference in the input space.
\par \textbf{Fault class smearing v.s. fault smearing:} Fault smearing~\cite{alcala2009reconstruction,westerhuis2000generalized} previously discussed in CP and RBC works is caused by the same weights shared by all the inputs, since the detection models only output one scalar (e.g., SPE fault index). Though the model weights of each fault class are unique, we show that the classifiers still weigh on the irrelevant variables (due to the task for distinguish different fault types). These non-zero irrelevant weights smear the variable contributions. In summary, fault class smearing demonstrates the phenomenon and cause of the fault smearing problem in fault classification.

\begin{figure}
	\centering
\includegraphics[width=0.48\textwidth]{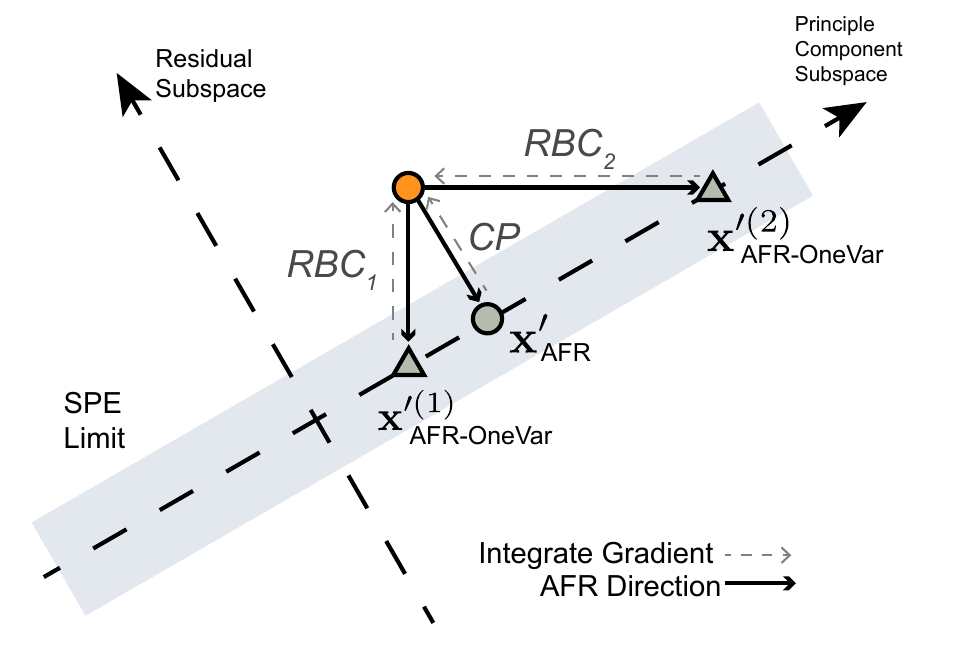}
\caption{Illustration of ABIGX on 2-D PCA-based detection: CP is identical to the integrated gradients between the explained sample and $\mathbf{x}'_{AFR}$; RBC is identical to the integrated gradients between the explained sample and $\mathbf{x}'_{AFR^{OV}}$ on each variable direction. }
\label{fig_abigx2d}
\end{figure}

\section{Fault Diagnosis: ABIGX Perspective}
\label{sec_FD}
\par This section presents how ABIGX uniformly reframes the two basic fault diagnosis methods, CP and RBC. For the PCA-based fault detection, we state that both CP and RBC are the specifications of ABIGX. Fig \ref{fig_abigx2d} intuitively shows the identicality of CP and RBC to ABIGX with two specific AFR-reconstructed samples, $\mathbf{x}'_{AFR}$ and $\mathbf{x}'_{AFR^{OV}}$.
\par In the following, the PCA-based detector is denoted by its residual subspace projection matrix, $\tilde{\mathbf{C}} = \mathbf{I}-\mathbf{P}\mathbf{P}^T$.

\subsection{RBC and ABIGX-OneVar}
\par Section \ref{sec_abigx} discusses that ABIGX is derived from the idea of RBC. Furthermore, we prove that RBC is a linear specification of ABIGX with the AFR-OneVar.

\begin{lemma}
	\label{prop_RBCpro}
	(Fault reconstruction properties~\cite{alcala2009reconstruction}) The value of $f_i$ in fault reconstruction such that $SPE(\mathbf{x}-f_i\xi_i)$ is minimized is:
	\begin{align}
		\label{eq_propfi}
		f_i = (\xi_i^T\tilde{\mathbf{C}}\xi_i)^{-1}\xi_i^T\tilde{\mathbf{C}}\mathbf{x}
	\end{align}

\end{lemma}
\begin{theorem}
\label{theo_rbc_ig} (Identicality between RBC and ABIGX) RBC is identical to ABIGX with one-variable AFR (i.e., ABIGX-OneVar):
\begin{align}
	RBC_i = ABIGX^{OV}_i
\end{align}

\end{theorem}
\begin{proof}
	 Recalling the ABIGX-OneVar definition in Eq. \ref{eq_ABIGX^{OV}} and IG definition in Eq. \ref{eq_IG}, the integrated gradient along each variable direction can be rewrote by integrating over reconstruction magnitude $f_i$:
	 \begin{align}
		ABIGX^{OV}_i &= - \int_{f_i}^{0} \frac{\partial f(\mathbf{x}-f\xi_i)}{\partial x_i}\,df \notag \\
			&=  \int_{0}^{f_i} 2\mathbf{\xi_i}^T \tilde{\mathbf{C}}(\mathbf{x}-f\xi_i)\,df \notag \\
			&= f_i\mathbf{\xi}_i^T \tilde{\mathbf{C}}\mathbf{x} \notag \\
			&= \mathbf{x}^T\tilde{\mathbf{C}}\xi_i(\xi_i^T\tilde{\mathbf{C}}\xi)_i^{-1}\mathbf{\xi}_i^T \tilde{\mathbf{C}}\mathbf{x}
	\end{align}
where $f_i$ is substituted by Eq. \ref{eq_propfi}. Similarly, substitute Eq. \ref{eq_propfi} to RBC expression (eq. \ref{eq_rbcdef}), the RBC contribution on $\xi_i$ is:
\begin{align}
	RBC_i &= \mathbf{x}^T\tilde{\mathbf{C}}\xi_i(\xi_i^T\tilde{\mathbf{C}}\xi)_i^{-1}\mathbf{\xi}_i^T \tilde{\mathbf{C}}  \mathbf{\xi}_i (\xi_i^T\tilde{\mathbf{C}}\xi)_i^{-1}\mathbf{\xi}_i^T \tilde{\mathbf{C}}\mathbf{x} \notag \\
	&=\mathbf{x}^T\tilde{\mathbf{C}}\xi_i(\xi_i^T\tilde{\mathbf{C}}\xi)_i^{-1}\mathbf{\xi}_i^T \tilde{\mathbf{C}}\mathbf{x}
\end{align}
which is identical to $ABIGX^{OV}_i$. From another perspective, ABIGX-OneVar contribution can be expressed by directly integrating over the sample $\mathbf{x}^*$:
	 \begin{align}
		ABIGX^{OV}_i &= \int_{\mathbf{x}-f_i\xi_i}^{\mathbf{x}} \frac{\partial SPE(\mathbf{x}^*)}{\partial x_i}\,d\mathbf{x}^* \notag \\
			&= SPE(\mathbf{x}) - SPE(\mathbf{x}-f_i\xi_i)
	\end{align}
which is consistent with the RBC property stated in Eq. 21 of RBC paper~\cite{alcala2009reconstruction}.
\end{proof}

\subsection{CP and ABIGX}
\par We prove that CP is also a linear specification of ABIGX with $\ell_2$ norm AFR. Firstly, we calculate the AFR-reconstructed sample for PCA under $\ell_2$ norm distance:
\begin{lemma} 
	\label{lem_afrpca}	
	(AFR for PCA) Under $\ell_2$ norm distance constraint, the AFR-reconstructed sample $\mathbf{x}'_{AFR}$ for PCA-based fault detection is:
	\begin{align}
		\label{eq_afrpca}
		\mathbf{x}'_{AFR} = \mathbf{x} - \tilde{\mathbf{C}}\mathbf{x}
	\end{align}
which means the reconstruction vector is exactly the residual subspace projection.
\end{lemma}
\begin{proof}
	Based on Definition \ref{def_afr}, since the minimal of SPE is always zero, we can solve $\mathbf{x}'_{AFR}$ by minimizing the reconstruction vector:
	\begin{align}
		\label{eq_cpafr}
		\mathbf{x}'_{AFR} = &\arg \min \Vert \mathbf{x} - \mathbf{x}'_{AFR}\Vert_2 \notag \\
		&s.t.\ SPE(\mathbf{x}'_{AFR}) =0
	\end{align}
\par We can construct the Lagrangian function by introducing Lagrange multiplier $\lambda$ (omit the subscript of $\mathbf{x}'_{AFR}$):
\begin{align}
	L(\mathbf{x}',\lambda) = ( \mathbf{x} - \mathbf{x}')^T (\mathbf{x} - \mathbf{x}') + \lambda \mathbf{x}'^T\tilde{\mathbf{C}}\mathbf{x}'
\end{align}
\par The minimal value is obtained when:
\begin{align}
	\begin{cases}
		\label{eq_constr}
		\nabla_{\mathbf{x}'} L(\mathbf{x}',\lambda) = 0 \\
		\mathbf{x}'^T\tilde{\mathbf{C}}\mathbf{x}' = 0
	\end{cases}
\end{align}
which turns to:
\begin{align}
	\label{eq_larg}
	\mathbf{x} - \mathbf{x}' = \lambda \tilde{\mathbf{C}}\mathbf{x}'\notag \\
	\mathbf{x}'^T (\mathbf{x} - \mathbf{x}') = 0
\end{align}
\par The reconstructed sample $\mathbf{x}'$ with zero SPE should be orthogonal to the residual subspace and Eq. \ref{eq_larg} implies the optimal reconstruction direction $\mathbf{x} - \mathbf{x}'$ is orthogonal to the reconstructed sample $\mathbf{x}'$. Hence, we can obtain the reconstruction vector is orthogonal to the principal component subspace, which is exactly the residual projection of explained sample, $\tilde{\mathbf{C}}\mathbf{x}$. 
\end{proof}
\par Subsequently, we can prove the identicality between CP and ABIGX.
\begin{theorem}
(Identicality between CP and ABIGX) CP is identical to ABIGX with $\ell_2$ norm AFR:
\begin{align}
	CP_i = ABIGX_i
\end{align}
\end{theorem}
\begin{proof}
	With $\mathbf{x}'_{AFR}$ in Lemma \ref{lem_afrpca}, the ABIGX is (omit the subscript of $\mathbf{x}'_{AFR}$):
	\begin{align}
		ABIGX_i &= \xi_i^T \tilde{\mathbf{C}}\mathbf{x} \int_{\alpha=0}^1 \frac{\partial SPE(\mathbf{x}' + \alpha\tilde{\mathbf{C}}\mathbf{x})}{\partial \mathbf{x}_i}\,d\alpha  \notag \\
		&= \xi_i^T \tilde{\mathbf{C}}\mathbf{x} \xi_i^T (2\tilde{\mathbf{C}}\mathbf{x}' +\tilde{\mathbf{C}}\mathbf{x})
	\end{align}
	\par From Eq. \ref{eq_afrpca}, we have $\tilde{\mathbf{C}}\mathbf{x}'= \tilde{\mathbf{C}} (\mathbf{x}- \tilde{\mathbf{C}}\mathbf{x})=0$, thus the ABIGX contribution is obtained:
	\begin{align}
	ABIGX_i = (\xi_i^T \tilde{\mathbf{C}}\mathbf{x})^2
\end{align}
which is identical to $CP_i$.
\end{proof}

\par Additional, inspired by Lemma \ref{lem_afrpca}, we discuss the feasibility of gradient-based attack, which is the default algorithm of AFR and ABIGX:
\begin{remark}
	(Gradient-based attack on PCA) Gradient-based attack always reconstructs fault in the direction of residual subspace projection. 
\end{remark} 
\begin{proof}
	The adversarial perturbation at each step of PGD is:
	\begin{align}
		\frac{1}{L} \nabla_\mathbf{x}SPE(\mathbf{x})=\frac{2}{L} \tilde{\mathbf{C}} \mathbf{x} \propto  \tilde{\mathbf{C}} \mathbf{x} 
	\end{align}
	where $L$ is a scalar.
\end{proof}
\par This is consistent with intuition since the residual projection points towards the direction that SPE descends fastest.

\begin{table}[]
	\centering
	\caption{Mean difference between variable contribution of CP, RBC, and ABIGX (on PCA detector for TEP).}
	\begin{tabular}{@{}lcc|cc@{}}
	\toprule
		& \multicolumn{2}{c|}{ABIGX} & \multicolumn{2}{c}{ABIGX-OneVer} \\
	AFR	& PGD     & MILP    & PGD  & MILP  \\\midrule
	CP  & \num{1.62e-6}   & \num{1.25e-6}     & -             & -              \\
	RBC & -           & -           & \num{8.87e-3}       & \num{3.5e-5}         \\ \bottomrule
	\end{tabular}
	\label{table_error}
\end{table}

\section{Experiments}
\label{sec_expr}
\subsection{Evaluation metrics and Compared methods}
\par We introduce two kinds of quantitative metrics for the fair evaluation of explanation methods:
\par \textbf{Correctness}~\cite{8315047} measures the difference between variable contribution and the ground-truth root cause. The better explanation should be closer to the ground-truth. Specifically, \emph{Correctness-AUC} treats the contributions as binary classification prediction scores. By changing the threshold of contribution scores to be negative class, the area under the receiver operating characteristic curve is calculated \cite{kapishnikov2021guided}. The correctness can also be measured by summing up the attribution scores on the ground-truth root cause variable, called \emph{Correctness-SUM}.

\par \textbf{Consistency}~\cite{kapishnikov2019xrai} tests whether contributions focus on where the model is truly looking. The better explanation methods should be more consistent with model behavior. The metric \emph{Consistency-ADD} gradually adds the variable values of the explained fault sample to the normality sample. By sliding a contribution threshold, the variable with the largest contribution is first added and the least the last. The better explanation method should increase the model prediction\footnote{The model prediction for classification is fault type confidence with Softmax function, and that for detection is SPE value. When calculating the AUC of detection models, we normalize all SPE values to the range [0,1].} more quickly, which can be quantized by the area under the prediction score curve w.r.t threshold. Conversely, \emph{Consistency-DEL} first deletes the most important variables until all the features are replaced by the normality sample. Similarly, the better method should decrease the model prediction faster, which can be measured by the AUC w.r.t. threshold.

\par For comparing with ABIGX, we introduce three advanced XAI methods, saliency map, IG and DeepLift. DeepLift~\cite{DBLP:journals/corr/ShrikumarGK17} is a gradient backpropagation explanation method that can properly handle cases where IG may give misleading results.

\subsection{Tennessee Eastman Process (TEP)}
\subsubsection{Dataset Description}
\par TEP is a standard benchmark that has been widely adopted by existing FDC and diagnosis works. We consider 33 measurement variables and 14 fault types of TEP dataset. The ground-truth root causes of 14 TEP faults are mainly based the detailed descriptions of variables and faults (see Table 8.1, 8.2, and 8.4 in \cite{chiang2000fault}) and also refers to some fault diagnosis works~\cite{qian2020locally,9852307,peng2022towards,doi:10.1021/acs.iecr.6b01916}. For clarity, we intuitively display the explanation comparisons for two TEP fault types, Fault 6 and Fault 14, respectively in Fig. \ref{fig_TEfault14} and \ref{fig_TEfault6}. Fault 6 is caused by root variables XMEAS(1) and XMV(3) and Fault 14 is caused by XMEAS(9), XMV(10), and XMEAS(21).
\par Fig. \ref{fig_TEfault14} and \ref{fig_TEfault6} respectively display sensor variable contributions of fault detection and classification models, where the four subplots of each figure show the variable contributions explained by four different techniques. Besides ABIGX, IG and saliency map, we test the CP techniques for detection model and DeepLift for classification model.
\par  Fig. \ref{fig_TEfault14} and \ref{fig_TEfault6} uses the plot style from the work~\cite{NIPS2017_7062}. The horizontal axis represents the variable contribution value (with direction). The variables are sorted by importance on the vertical axis, where the eight variables with the most contributions are displayed. Notably, the plots also show the correlation between variable (feature) value and its contribution by different colored points (a point in the row represents a fault sample).

	\begin{figure}[t]
		\centering
		\includegraphics[width=0.5\textwidth]{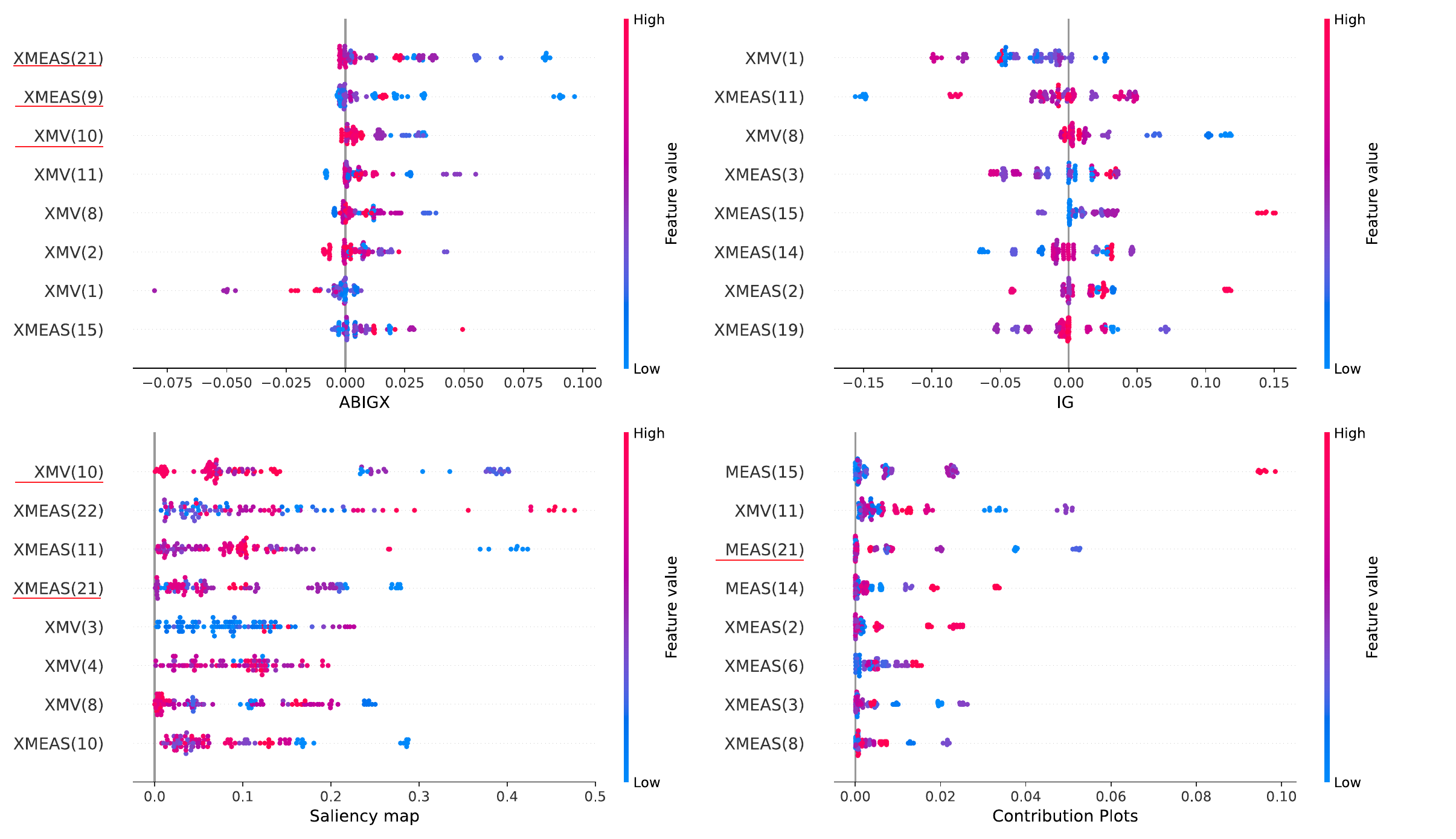}
		\caption{Explanation comparison for AE detector on TEP Fault 14, where the ground-truth root variables are underlined in red.}
	\label{fig_TEfault14}
	\end{figure}

\subsubsection{Explainable Fault Detection}
\par \textbf{Identicality between Fault Diagnosis and ABIGX:} Section \ref{sec_FD} theoretically proves the identicality between two fault diagnosis methods and ABIGX. On the PCA detector for TEP faults, we compute the empirical error between them, which is the mean value over contributions of all variables. Table \ref{table_error} reports the results, which shows the differences are small enough (especially with AFR-reconstructed samples via MILP).

\begin{table}[]
	\renewcommand\arraystretch{1.4}
	\centering
	\caption{Evaluation metrics of explanation methods (on AE detector for TEP).}
	\begin{tabular}{@{}lcc|cc@{}}
	\toprule
				 & \multicolumn{2}{c|}{Correctness} & \multicolumn{2}{c}{Consistency} \\
	Metrics			 & AUC$\uparrow$             & SUM$\uparrow$            & ADD$\uparrow$            & DEL$\downarrow$            \\  \midrule
	ABIGX        & 0.382          & 0.041          & \textbf{0.922}          & \textbf{0.256}          \\
	ABIGX-OneVar & \textbf{0.478}   & \textbf{0.055}          & 0.903          & 0.275          \\
	CP           & 0.326          & 0.028          & 0.896          & 0.283          \\
	Saliency map & 0.470          & 0.051          & 0.897          & 0.292          \\
	IG           & 0.395          & 0.045          & 0.913          & 0.267          \\ \bottomrule
	\end{tabular}
	\label{tbl_AEmetrics}
\end{table}

\textbf{Explanation results:} Firstly, Fig \ref{fig_TEfault14} plots the intuitive comparisons between the different explanations on AE models for the TEP dataset, showing that ABIGX assigns the highest contributions to the root variables. At the same time, other methods do not provide complete explanations. 

\par Table \ref{tbl_AEmetrics} reports four evaluation metrics on the explanations of AE detection for TEP, where 14 Fault types are concerned. Overall, ABIGX methods outperform other explanations, including the advanced IG (DeepLift is not adapted for fault detection). ABIGX-OneVar provides the explanations that are closest to the ground-truth root variables. The saliency map that attributes variables independently also performs well on the correctness metrics. On the other hand, ABIGX provides explanations that best match the model behavior. We argue that the explanations concerning variable interaction are closer to how models actually deal with the samples.

\subsubsection{Explainable Fault Classification}
\par Similar to fault detection, Fig \ref{fig_TEfault6} plots the explanation of NN-based fault classifier on TEP Fault 6. The advantage of ABIGX is on Variable XMV(3). ABIGX assigns higher contributions to this variable, while the saliency map even ignores it, the reason for which is analyzed in Section \ref{sec_fcinsight}.

\par Table \ref{tbl_NNmetrics} reports the explanation metrics on fault classification, where the method performances are similar to the detection. ABIGX still generally performs better, even under the correctness (except AUC). ABIGX-OneVar still emphasizes correctness but has poor model consistency, which is similar to the saliency map. DeepLift and IG have similar performances. 

\par \textbf{Validation of classification SPE index:} An ablation study is performed with ABIGX-AdvAFR. Instead of ABIGX's classification SPE proposed in Section \ref{sec_spefc}, ABIGX-AdvAFR reconstructs faults by the original adversarial attack objective (i.e., minimizing Softmax output confidences of normal class). The overall better performance of ABIGX than ABIGX-AdvAFR validates the effectiveness of the novel classification SPE index for fault reconstruction.

\begin{figure}
	\centering
\includegraphics[width=0.5\textwidth]{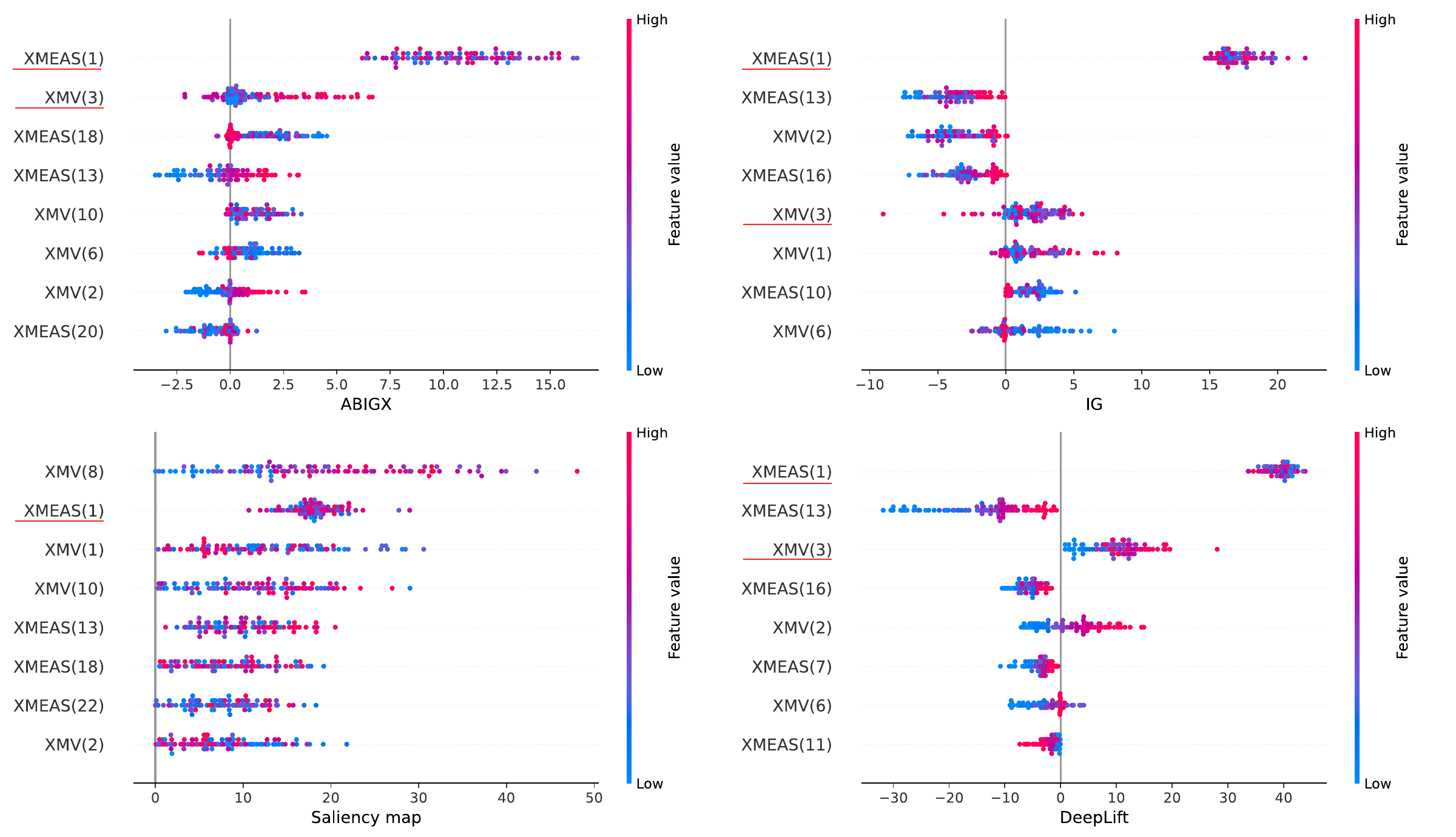}
\caption{Explanations comparison for NN classifier on TEP Fault 6, where the ground-truth root variables are underlined in red.}
\label{fig_TEfault6}
\end{figure}

\begin{table}[]
	\renewcommand\arraystretch{1.4}
	\centering
	\caption{Evaluation metrics of explanation methods (on NN classifier for TEP).}
	\begin{tabular}{@{}lcc|cc@{}}
	\toprule
				 & \multicolumn{2}{c|}{Correctness} & \multicolumn{2}{c}{Consistency} \\
	Metrics			 & AUC$\uparrow$             & SUM$\uparrow$            & ADD$\uparrow$            & DEL$\downarrow$            \\  \midrule
	ABIGX        & 0.557          & \textbf{0.075}          & \textbf{0.890}          & \textbf{0.032}             \\
	ABIGX-OneVar & 0.584          & 0.073          & 0.723         & 0.121        \\
	Saliency map   & 0.573          & 0.072          & 0.732          & 0.133          \\
	DeepLift  & 0.400          & 0.037          & 0.876          & 0.042          \\
	IG           & 0.415          & 0.045          & 0.886          & 0.043          \\ 
	ABIGX-AdvAFR &\textbf{0.608} & 0.074&  0.726 &0.135 \\\bottomrule
	\end{tabular}	
	\label{tbl_NNmetrics}
	\end{table}

	\begin{figure}[t]
		\centering
		\includegraphics[width=0.44\textwidth]{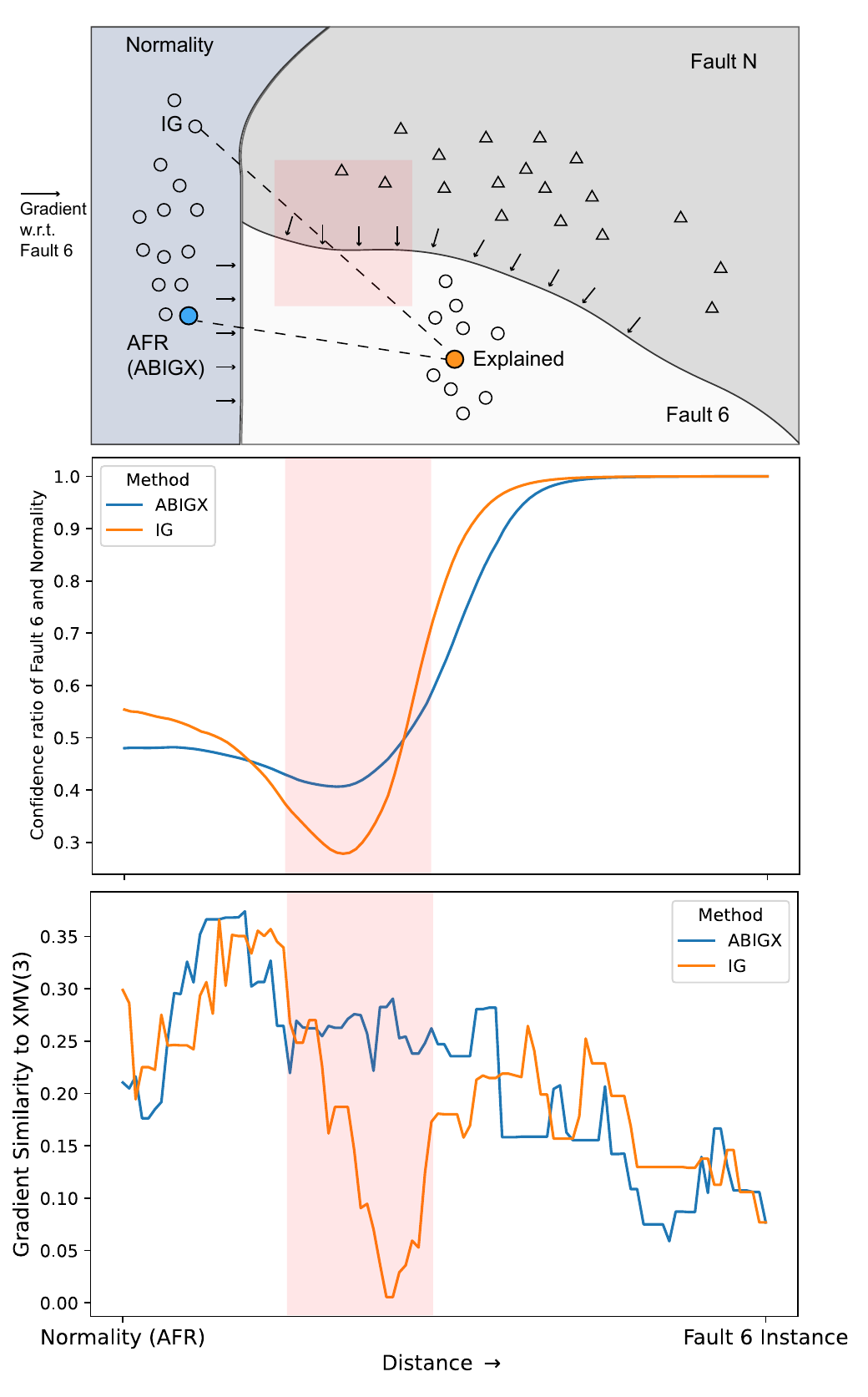}
		\caption{Effect of different baselines (ABIGX and IG) on the integrated gradients. Top) The intuitive illustration of two baselines and integral paths (not the real cases); Middle) The model confidences along the two paths from two baselines to explained sample; Bottom) The cosine similarity between gradient direction and the root variable XMV(3), along the same paths. The red boxed in three figures are noisy areas corresponding to each other.}
		\label{fig_insight}
	\end{figure}
	
\subsection{Insights}
\label{sec_fcinsight}
\par Besides the comparison between different reconstructed fault samples, we give a deep insight into the distinctions between ABIGX and IG, to explain why our proposal is superior for the XFDC.

\par The essential distinction between ABIGX and IG is the baseline (i.e., the start point of the integral interval): the baseline of ABIGX is the AFR-reconstructed sample, and IG's baseline is the normality sample. We discuss the effect of these two baselines on the variable contribution results. The explanations of TEP Fault 6 in Fig. \ref{fig_TEfault6} show that the major advantage of ABIGX to IG is Variable XMV(3). Hence, Fig. \ref{fig_insight} gives a deep insight into the reason that causes different contributions of XMV(3).

\par We argue that the two different baselines lead to different paths. The path of IG is more likely to go through the area where other fault types (i.e., fault types other than explained fault and normality) have high confidence, which we call noisy area. The red boxed in Fig. \ref{fig_insight} present this area. ABIGX is more likely to avoid the noisy area due to its baseline by AFR, which reconstructs faults on the direction that fastest minimizes the classification SPE. Hence, we suppose that the path of ABIGX is more related to the classification areas of normality and explained fault (e.g., Fault 6). On the contrary, the baseline of the random normality sample may lead IG to the noisy area, the possible case of which is depicted in the top of Fig. \ref{fig_insight}.

\par To validate our analysis, Fig. \ref{fig_insight} also reports the case of a TEP sample with Fault 6. Firstly, the middle figure reports the confidence ratio of (the sum of) Fault 6 and normality. When moving from the baseline to explained sample, IG goes through an area where the confidence ratio drops rapidly (red box), which means the confidence of other faults increases largely (e.g., Fault N in the top of Fig. \ref{fig_insight}). Secondly, the bottom of Fig. \ref{fig_insight} reports the gradient similarity to the root variable XMV(3) along the paths. The noisy area gradient is completely irrelevant to XMV(3). We can find that these two areas are heavily overlapped.

\par Unlike detection, fault classification models the distinction between normality and fault, and also distinguishes different fault types. However, XFDC explains what variables cause the model to predict the fault sample as a ``fault'' rather than normality. As for why Fault 6 is not predicted as the ``Fault N'', this is not considered by XFDC. Hence, this is why the noisy area leads to incorrect gradient direction (e.g., the gradient directions in the top of Fig. \ref{fig_insight}).

\par Incidentally, the gradient plot in the bottom of Fig. \ref{fig_insight} explains why the saliency map is inaccurate: the saliency map only computes the gradient on the explained sample, but the gradient of XMV(3) is much higher when closing to the normality (or AFR-reconstruction).

\begin{figure}
	\centering
	\includegraphics[width=0.45\textwidth]{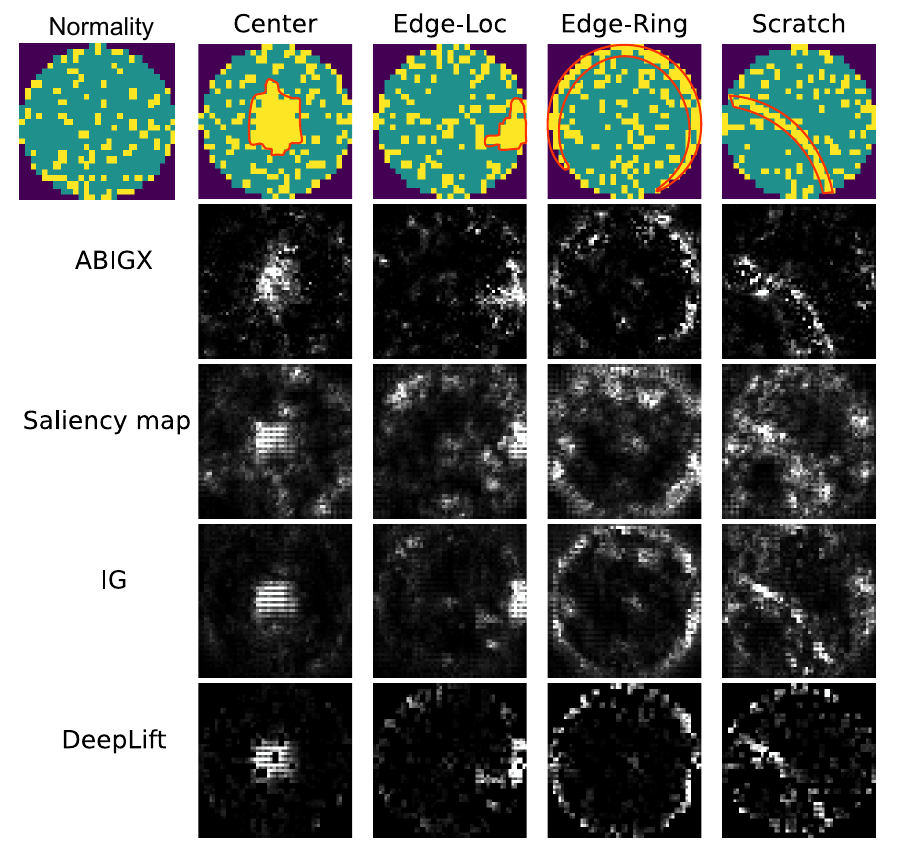}
	\caption{Explanations for CNN classifier on four sampled wafer maps with different fault patterns, where the ground-truth root pixels are circled in red.}
	\label{fig_wm}
\end{figure}

\subsection{Wafer Map Fault}
\subsubsection{Dataset Description}
\par Wafer map analysis is critical in daily semiconductor manufacturing operations. Wafer maps provide visual details that are crucial for identifying the stage of manufacturing at which wafer fault pattern occurs. We study the explanation of the wafer map fault pattern classification, the target of which is to determine which pixels of the wafer maps are the root cause of faults.

\par We train a convolution neural network (CNN)-based fault classifier on WM-811K dataset~\cite{6932449}, of which four fault types (Center, Edge-Loc, Edge-Ring, and Scratch) are selected. The fault and normal wafer maps are displayed in Fig. \ref{fig_wm}, where the ground-truth root pixels are circled by red lines. The green pixels denote the normal dies, and the yellow pixels are the defective dies. A wafer map is normal when there is no certain defective die pattern (e.g., the top left one in Fig. \ref{fig_wm}).
\subsubsection{Explainable Fault Classification}
\par Fig \ref{fig_wm} shows the explanations comparison on four samples. Compared with the saliency map and IG, the variable contributions of ABIGX are less noisy and focus more on the root causes. DeepLift also provides clear variable contributions, but some of them are incomplete (e.g., Center and Edge-Loc fault type), and some are over attributed (e.g., the ring of Scratch fault). Table \ref{tbl_wmmetrics} reports the four introduced metrics to evaluate the performances on the wafer map dataset, showing that ABIGX still provides a general better explanation. 
\par The experiments on wafer map fault explanation validates the generalization of ABIGX, which is not only applicable for sensor data but also for the unstructured images trained by CNN models.

\begin{table}[]
	\renewcommand\arraystretch{1.4}
	\centering
	\caption{Evaluation metrics of explanation methods (on CNN classifier for wafer maps).}
	\begin{tabular}{@{}lcc|cc@{}}
	\toprule
				 & \multicolumn{2}{c|}{Correctness} & \multicolumn{2}{c}{Consistency} \\
	Metrics			 & AUC$\uparrow$             & SUM$\uparrow$            & ADD$\uparrow$            & DEL$\downarrow$            \\  \midrule
	ABIGX        & \textbf{0.839}          & 0.524          & \textbf{0.859}          & \textbf{0.040}             \\
	Saliency map   & 0.571          & 0.298         & 0.623          & 0.236          \\
	IG & 0.814          & 0.477          & 0.796          & 0.052          \\
	DeepLift            & 0.812          & \textbf{0.541}          & 0.806          & 0.053           \\\bottomrule
	\end{tabular}	
	\label{tbl_wmmetrics}
	\end{table}

\section{Conclusion}
\label{sec_conclu}
\par This work proposes a unified framework, ABIGX, for explaining the general FDC models. ABIGX is derived from the essentials of CP and RBC, which have succeeded in fault diagnosis of PCA models. The novel proposed fault reconstruction method, AFR, is the core part of ABIGX. AFR reframes the existing methods from the adversarial attack perspective and generalizes FR to the fault classification by the novel classification SPE index. For explainable fault classification, we analyze the performances in the problem of fault class smearing. Compared with the saliency map and IG, we prove that ABIGX provides more accurate variable contributions with a lower degree of fault class smearing. For explainable fault detection, we theoretically prove the identicality of linear CP and RBC to the ABIGX. With quantitative metrics, the experiment on a wide range of models and datasets validates the general superiority of ABIGX to other advanced explainers.
\section*{Acknowledgement}
This work was supported in part by the National Natural Science Foundation of China (NSFC) (92167106).
\bibliographystyle{IEEEtran}
\bibliography{manuscript}

\begin{thebibliography}{10}
\providecommand{\url}[1]{#1}
\csname url@samestyle\endcsname
\providecommand{\newblock}{\relax}
\providecommand{\bibinfo}[2]{#2}
\providecommand{\BIBentrySTDinterwordspacing}{\spaceskip=0pt\relax}
\providecommand{\BIBentryALTinterwordstretchfactor}{4}
\providecommand{\BIBentryALTinterwordspacing}{\spaceskip=\fontdimen2\font plus
\BIBentryALTinterwordstretchfactor\fontdimen3\font minus \fontdimen4\font\relax}
\providecommand{\BIBforeignlanguage}[2]{{%
\expandafter\ifx\csname l@#1\endcsname\relax
\typeout{** WARNING: IEEEtran.bst: No hyphenation pattern has been}%
\typeout{** loaded for the language `#1'. Using the pattern for}%
\typeout{** the default language instead.}%
\else
\language=\csname l@#1\endcsname
\fi
#2}}
\providecommand{\BIBdecl}{\relax}
\BIBdecl

\bibitem{5282515}
I.~Hwang, S.~Kim, Y.~Kim, and C.~E. Seah, ``A survey of fault detection, isolation, and reconfiguration methods,'' \emph{IEEE Transactions on Control Systems Technology}, vol.~18, no.~3, pp. 636--653, 2010.

\bibitem{9430765}
J.~Li, D.~Ding, and F.~Tsung, ``Directional pca for fast detection and accurate diagnosis: A unified framework,'' \emph{IEEE Transactions on Cybernetics}, pp. 1--11, 2021.

\bibitem{9632460}
C.~Huang, Z.~Yang, J.~Wen, Y.~Xu, Q.~Jiang, J.~Yang, and Y.~Wang, ``Self-supervision-augmented deep autoencoder for unsupervised visual anomaly detection,'' \emph{IEEE Transactions on Cybernetics}, pp. 1--14, 2021.

\bibitem{yadav2014overview}
A.~Yadav and Y.~Dash, ``An overview of transmission line protection by artificial neural network: fault detection, fault classification, fault location, and fault direction discrimination,'' \emph{Advances in Artificial Neural Systems}, vol. 2014, 2014.

\bibitem{surveyXAI}
X.-H. Li, C.~C. Cao, Y.~Shi, W.~Bai, H.~Gao, L.~Qiu, C.~Wang, Y.~Gao, S.~Zhang, X.~Xue, and L.~Chen, ``A survey of data-driven and knowledge-aware explainable ai,'' \emph{IEEE Transactions on Knowledge and Data Engineering}, vol.~34, no.~1, pp. 29--49, 2022.

\bibitem{miller1993contribution}
P.~Miller, R.~Swanson, and C.~Heckler, ``Contribution plots: The missing link in multivariate quality control,'' \emph{J. Qual. Tech}, 1993.

\bibitem{alcala2009reconstruction}
C.~F. Alcala and S.~J. Qin, ``Reconstruction-based contribution for process monitoring,'' \emph{Automatica}, vol.~45, no.~7, pp. 1593--1600, 2009.

\bibitem{FR}
R.~Dunia and S.~Joe~Qin, ``Subspace approach to multidimensional fault identification and reconstruction,'' \emph{AICHE journal}, vol.~44, no.~8, pp. 1813--1831, 1998.

\bibitem{hallgrimsson2020improved}
{\'A}.~D. Hallgr{\'\i}msson, H.~H. Niemann, and M.~Lind, ``Improved process diagnosis using fault contribution plots from sparse autoencoders,'' \emph{IFAC-PapersOnLine}, vol.~53, no.~2, pp. 730--737, 2020.

\bibitem{qian2020locally}
J.~Qian, L.~Jiang, and Z.~Song, ``Locally linear back-propagation based contribution for nonlinear process fault diagnosis,'' \emph{IEEE/CAA Journal of Automatica Sinica}, vol.~7, no.~3, pp. 764--775, 2020.

\bibitem{tan2019multi}
R.-M. Tan and Y.~Cao, ``Multi-layer contribution propagation analysis for fault diagnosis,'' \emph{International Journal of Automation and Computing}, vol.~16, no.~1, pp. 40--51, 2019.

\bibitem{alcala2010reconstruction}
C.~F. Alcala and S.~J. Qin, ``Reconstruction-based contribution for process monitoring with kernel principal component analysis,'' \emph{Industrial \& Engineering Chemistry Research}, vol.~49, no.~17, pp. 7849--7857, 2010.

\bibitem{deng2020sparse}
Z.~Deng, Y.~Li, H.~Zhu, K.~Huang, Z.~Tang, and Z.~Wang, ``Sparse stacked autoencoder network for complex system monitoring with industrial applications,'' \emph{Chaos, Solitons \& Fractals}, vol. 137, p. 109838, 2020.

\bibitem{goodfellow2014explaining}
I.~J. Goodfellow, J.~Shlens, and C.~Szegedy, ``Explaining and harnessing adversarial examples,'' \emph{arXiv preprint arXiv:1412.6572}, 2014.

\bibitem{simonyan2013deep}
K.~Simonyan, A.~Vedaldi, and A.~Zisserman, ``Deep inside convolutional networks: Visualising image classification models and saliency maps,'' \emph{arXiv preprint arXiv:1312.6034}, 2013.

\bibitem{sundararajan2017axiomatic}
M.~Sundararajan, A.~Taly, and Q.~Yan, ``Axiomatic attribution for deep networks,'' in \emph{International conference on machine learning}.\hskip 1em plus 0.5em minus 0.4em\relax PMLR, 2017, pp. 3319--3328.

\bibitem{nomikos1995multivariate}
P.~Nomikos and J.~F. MacGregor, ``Multivariate spc charts for monitoring batch processes,'' \emph{Technometrics}, vol.~37, no.~1, pp. 41--59, 1995.

\bibitem{chiang2000fault}
L.~H. Chiang, E.~L. Russell, and R.~D. Braatz, \emph{Fault detection and diagnosis in industrial systems}.\hskip 1em plus 0.5em minus 0.4em\relax Springer Science \& Business Media, 2000.

\bibitem{joe2003statistical}
S.~Joe~Qin, ``Statistical process monitoring: basics and beyond,'' \emph{Journal of Chemometrics: A Journal of the Chemometrics Society}, vol.~17, no. 8-9, pp. 480--502, 2003.

\bibitem{8386786}
V.~L. Cao, M.~Nicolau, and J.~McDermott, ``Learning neural representations for network anomaly detection,'' \emph{IEEE Transactions on Cybernetics}, vol.~49, no.~8, pp. 3074--3087, 2019.

\bibitem{9199886}
L.~Vu, V.~L. Cao, Q.~U. Nguyen, D.~N. Nguyen, D.~T. Hoang, and E.~Dutkiewicz, ``Learning latent representation for iot anomaly detection,'' \emph{IEEE Transactions on Cybernetics}, vol.~52, no.~5, pp. 3769--3782, 2022.

\bibitem{abid2021review}
A.~Abid, M.~T. Khan, and J.~Iqbal, ``A review on fault detection and diagnosis techniques: basics and beyond,'' \emph{Artificial Intelligence Review}, vol.~54, no.~5, pp. 3639--3664, 2021.

\bibitem{dong2019quality}
J.~Dong, R.~Sun, K.~Peng, Z.~Shi, and L.~Ma, ``Quality monitoring and root cause diagnosis for industrial processes based on lasso-sae-cca,'' \emph{IEEE Access}, vol.~7, pp. 90\,230--90\,242, 2019.

\bibitem{ren2018new}
S.~Ren, F.~Si, J.~Zhou, Z.~Qiao, and Y.~Cheng, ``A new reconstruction-based auto-associative neural network for fault diagnosis in nonlinear systems,'' \emph{Chemometrics and Intelligent Laboratory Systems}, vol. 172, pp. 118--128, 2018.

\bibitem{sturmfels2020visualizing}
P.~Sturmfels, S.~Lundberg, and S.-I. Lee, ``Visualizing the impact of feature attribution baselines,'' \emph{Distill}, vol.~5, no.~1, p. e22, 2020.

\bibitem{9852307}
Y.~Zhuo, Z.~Yin, and Z.~Ge, ``Attack and defense: Adversarial security of data-driven fdc systems,'' \emph{IEEE Transactions on Industrial Informatics}, pp. 1--16, 2022.

\bibitem{CEAE_3}
K.~Browne and B.~Swift, ``Semantics and explanation: why counterfactual explanations produce adversarial examples in deep neural networks,'' \emph{arXiv preprint arXiv:2012.10076}, 2020.

\bibitem{CEAE_4}
M.~Pawelczyk, C.~Agarwal, S.~Joshi, S.~Upadhyay, and H.~Lakkaraju, ``Exploring counterfactual explanations through the lens of adversarial examples: A theoretical and empirical analysis,'' in \emph{International Conference on Artificial Intelligence and Statistics}.\hskip 1em plus 0.5em minus 0.4em\relax PMLR, 2022, pp. 4574--4594.

\bibitem{lomuscio2017approach}
A.~Lomuscio and L.~Maganti, ``An approach to reachability analysis for feed-forward relu neural networks,'' \emph{arXiv:1706.07351}, 2017.

\bibitem{tjeng2018evaluating}
V.~Tjeng, K.~Y. Xiao, and R.~Tedrake, ``Evaluating robustness of neural networks with mixed integer programming,'' in \emph{International Conference on Learning Representations}, 2019.

\bibitem{cheng2017maximum}
C.-H. Cheng, G.~N{\"u}hrenberg, and H.~Ruess, ``Maximum resilience of artificial neural networks,'' in \emph{International Symposium on Automated Technology for Verification and Analysis}.\hskip 1em plus 0.5em minus 0.4em\relax Springer, 2017, pp. 251--268.

\bibitem{PGD}
A.~Madry, A.~Makelov, L.~Schmidt, D.~Tsipras, and A.~Vladu, ``Towards deep learning models resistant to adversarial attacks,'' \emph{arXiv preprint arXiv:1706.06083}, 2017.

\bibitem{IG}
M.~Sundararajan, A.~Taly, and Q.~Yan, ``Axiomatic attribution for deep networks,'' in \emph{International conference on machine learning}.\hskip 1em plus 0.5em minus 0.4em\relax PMLR, 2017, pp. 3319--3328.

\bibitem{liu2002gabor}
C.~Liu and H.~Wechsler, ``Gabor feature based classification using the enhanced fisher linear discriminant model for face recognition,'' \emph{IEEE Transactions on Image processing}, vol.~11, no.~4, pp. 467--476, 2002.

\bibitem{westerhuis2000generalized}
J.~A. Westerhuis, S.~P. Gurden, and A.~K. Smilde, ``Generalized contribution plots in multivariate statistical process monitoring,'' \emph{Chemometrics and intelligent laboratory systems}, vol.~51, no.~1, pp. 95--114, 2000.

\bibitem{8315047}
Z.~Bylinskii, T.~Judd, A.~Oliva, A.~Torralba, and F.~Durand, ``What do different evaluation metrics tell us about saliency models?'' \emph{IEEE Transactions on Pattern Analysis and Machine Intelligence}, vol.~41, no.~3, pp. 740--757, 2019.

\bibitem{kapishnikov2021guided}
A.~Kapishnikov, S.~Venugopalan, B.~Avci, B.~Wedin, M.~Terry, and T.~Bolukbasi, ``Guided integrated gradients: An adaptive path method for removing noise,'' in \emph{Proceedings of the IEEE/CVF conference on computer vision and pattern recognition}, 2021, pp. 5050--5058.

\bibitem{kapishnikov2019xrai}
A.~Kapishnikov, T.~Bolukbasi, F.~Vi{\'e}gas, and M.~Terry, ``Xrai: Better attributions through regions,'' in \emph{Proceedings of the IEEE/CVF International Conference on Computer Vision}, 2019, pp. 4948--4957.

\bibitem{DBLP:journals/corr/ShrikumarGK17}
\BIBentryALTinterwordspacing
A.~Shrikumar, P.~Greenside, and A.~Kundaje, ``Learning important features through propagating activation differences,'' \emph{CoRR}, vol. abs/1704.02685, 2017. [Online]. Available: \url{http://arxiv.org/abs/1704.02685}
\BIBentrySTDinterwordspacing

\bibitem{peng2022towards}
P.~Peng, Y.~Zhang, H.~Wang, and H.~Zhang, ``Towards robust and understandable fault detection and diagnosis using denoising sparse autoencoder and smooth integrated gradients,'' \emph{ISA transactions}, vol. 125, pp. 371--383, 2022.

\bibitem{doi:10.1021/acs.iecr.6b01916}
H.~Gharahbagheri, S.~A. Imtiaz, and F.~Khan, ``Root cause diagnosis of process fault using kpca and bayesian network,'' \emph{Industrial \& Engineering Chemistry Research}, vol.~56, no.~8, pp. 2054--2070, 2017.

\bibitem{NIPS2017_7062}
\BIBentryALTinterwordspacing
S.~M. Lundberg and S.-I. Lee, ``A unified approach to interpreting model predictions,'' in \emph{Advances in Neural Information Processing Systems 30}, I.~Guyon, U.~V. Luxburg, S.~Bengio, H.~Wallach, R.~Fergus, S.~Vishwanathan, and R.~Garnett, Eds.\hskip 1em plus 0.5em minus 0.4em\relax Curran Associates, Inc., 2017, pp. 4765--4774. [Online]. Available: \url{http://papers.nips.cc/paper/7062-a-unified-approach-to-interpreting-model-predictions.pdf}
\BIBentrySTDinterwordspacing

\bibitem{6932449}
M.-J. Wu, J.-S.~R. Jang, and J.-L. Chen, ``Wafer map failure pattern recognition and similarity ranking for large-scale data sets,'' \emph{IEEE Transactions on Semiconductor Manufacturing}, vol.~28, no.~1, pp. 1--12, 2015.

\end{thebibliography}

\appendix
\section{Proofs}
This section provides the proofs of the some results in the main paper.
\subsection{Proof of Theorem \ref{theo_optweight}}
\label{append_pr1}
\begin{manualtheorem}{\ref{theo_optweight}}[Restated]
		Let $w_{y,i}\in \mathbf{W}^*$ be the optimal model weight for variable $i$ on fault class $y$. For normality class $y=0$, the optimal $w_{0,i}$ is:
	\begin{align}
		w_{0,i} = \begin{cases}
			-\frac{1}{N}, & i \in \{1,\cdots,N\}\\
			0, & i \in \{N+1,\cdots,M\}
		\end{cases}
	\end{align}
	\par For fault class $y$, the optimal $w_{y,i}$ is: 
	\begin{align}
		w_{y,i} = \begin{cases}
			\frac{N}{2N-1},&\ i = y \\ 
			-\frac{1}{2N-1}, & i \in \{1,\cdots,N\}\setminus y \\
			0, & i \in \{N+1,\cdots,M\}
		\end{cases}
	\end{align}
\end{manualtheorem}

\begin{proof}
	We can obtain optimal weight $\mathbf{w}_y^*$ for class $y$, by calculating the derivate of Eq. \ref{eq_fld} w.r.t. $\mathbf{w}_y$:
	\begin{align}
	\nabla_{\mathbf{w}_y} = \frac{1}{2}\frac{(\mathbf{w}_y^T \Sigma \mathbf{w}_y)\Sigma_b \mathbf{w}_y - (\mathbf{w}_y^T \Sigma_b \mathbf{w}_y)\Sigma \mathbf{w}_y }{(\mathbf{w}_y^T \Sigma \mathbf{w}_y)^2}
	\end{align}
	where $\Sigma_b = (\mathbf{\mu}_y - \overline{\mathbf{\mu}})(\mathbf{\mu}_y - \overline{\mathbf{\mu}})^T$. The optimal $\mathbf{w}_y^*$ can be found when derivate equals zero:
\begin{align}
	0 &= (\mathbf{w}_y^T \Sigma \mathbf{w}_y)\Sigma_b \mathbf{w}_y - (\mathbf{w}_y^T \Sigma_b \mathbf{w}_y)\Sigma \mathbf{w}_y \notag \\
	 \mathbf{w}_y &= \Sigma^{-1} \frac{(\mathbf{w}_y^T \Sigma \mathbf{w}_y)\Sigma_b \mathbf{w}_y}{\mathbf{w}_y^T \Sigma_b \mathbf{w}_y} \notag\\
	&=\frac{\mathbf{w}_y^T \Sigma \mathbf{w}_y}{\mathbf{w}_y^T \Sigma_b \mathbf{w}_y} (\mathbf{\mu}_y - \overline{\mathbf{\mu}})^T \mathbf{w}_y  \Sigma^{-1} (\mathbf{\mu}_y - \overline{\mathbf{\mu}}) 
\end{align}
where the term $\frac{\mathbf{w}_y^T \Sigma \mathbf{w}_y}{\mathbf{w}_y^T \Sigma_b \mathbf{w}_y} (\mathbf{\mu}_y - \overline{\mathbf{\mu}})^T \mathbf{w}_y$ is a scalar and it can be omitted since $\mathbf{w}_y$ will be normalized. Hence, the direction of $\mathbf{w}_y$ is determined by:
\begin{align}
	 \mathbf{w}_y = k\cdot \Sigma^{-1} (\mathbf{\mu}_y - \overline{\mathbf{\mu}})
\end{align}
\par Since Example \ref{examp_data} defines the fault and normal variables have the same variances, the $\Sigma^{-1}=k\cdot J$ is a matrix with all the same elements. Thus, we can obtain that the weight of each class is in the same direction as the difference between the class mean and the overall mean. 
\par The normality class has the identical mean difference for the first $N$ variables. The mean difference of fault class $y$ is:

\begin{align}
	\label{eq_variablem}
	(\mathbf{\mu}_y - \overline{\mathbf{\mu}})_i=
	\begin{cases}
	 f,&\ i = y \\ 
	- \frac{f}{N},&\ i \in \{1,\cdots,N\}\setminus y \\
	 0,& \  i \in \{N+1,M\}
	\end{cases}
\end{align}
\par With $\ell_1$ normalization, we can obtain the optimal classifier weight $\mathbf{W}^*$ by stacking the optimal row vectors, which can be expressed in the matrix form:

\begin{align}\mathbf{W}^* = \left.
\begin{bNiceMatrix}
	w_0  & w_0  & \cdots & w_0  & 0 &  \cdots & 0 \\
	w_1 & w_2  & \cdots & w_2    & 0 &  \cdots & 0 \\
	w_2 & w_1  & \cdots & w_2    & 0 &  \cdots & 0 \\
	\cdots & \cdots   & \ddots & \cdots   & \cdots & \ddots & \cdots \\
	w_2 & w_2   & \cdots & w_1   & 0 &  \cdots & 0 
	\CodeAfter
  \OverBrace{1-1}{1-4}{N}[shorten,yshift=1mm]
  \OverBrace{1-5}{1-7}{M-N}[shorten,yshift=1mm]
\end{bNiceMatrix}
\right\}N+1
\label{eq_matrix}
\end{align}

where $w_0 = -\frac{1}{N}$, $w_1 = \frac{N}{2N-1}$, and $w_2 = -\frac{1}{2N-1}$.
\end{proof}

\subsection{Proof of Theorem \ref{theo_fcsig}}
\label{append_pr2}
\begin{manualtheorem}{\ref{theo_fcsig}}[Restated]
	(Fault class smearing in IG) Given the explained fault type $y$ with the ground-truth variable $y$ and the baseline of the expectation of normality samples $\mathbf{x}'=\overrightarrow{0} $, the expectation of IG contributes are:
	\begin{align}
		\mathbb{E}_{\mathbf{x}\sim \mathbf{X}_y} IG_i =\begin{cases}
			\vert fw_1 \vert, &\ i = y \\ 
			\vert\frac{\sigma\sqrt{2}}{\sqrt{\pi}} w_2 \vert, & i \in \{1,\cdots,N\}\setminus y \\
		\end{cases}
	\end{align}
	\par The degree of fault class smearing ($FCS$) in IG is:
	\begin{align}
	FCS_{IG} = \frac{N-1}{N} \frac{\sigma\sqrt{2} }{f\sqrt{\pi}} < FCS_{grad}
	\end{align}
\end{manualtheorem}
\begin{proof}
	Assume the fault magnitude is far greater than variance, i.e., $f\gg \sigma$. Then the expectation of IG contribution on ground-truth variable $y$ is:
	\begin{align}
		\mathbb{E}_{\mathbf{x}\sim \mathbf{X}_y} IG_y =\mathbb{E}_{x_y\sim  \mathcal{N}(f,\sigma^2)} \vert \int_{x=0}^{x=x_y} w_1 \, dx \vert = \vert fw_1 \vert
	\end{align}
	where we assume $x_y$ is always greater than zero. The contributions of the other irrelevant variables ($i\neq y$) are: 
	\begin{align}
		\mathbb{E}_{\mathbf{x}\sim \mathbf{X}_y} IG_i& =\mathbb{E}_{x_i\sim  \mathcal{N}(0,\sigma^2)} \vert \int_{x=0}^{x=x_i} w_2 \, dx \vert \notag \\
		& = \mathbb{E}_{x_i\sim  \mathcal{N}(0,\sigma^2I)} \vert x_i \vert \vert w_2 \vert \notag \\
		& = \vert \frac{\sigma\sqrt{2}}{\sqrt{\pi}} w_2 \vert
	\end{align}
	where $\frac{\sigma\sqrt{2}}{\sqrt{\pi}}$ is the expectation of half-Gaussian distribution.
\end{proof}

\subsection{Proof of Theorem \ref{theo_abigx}}
\label{append_pr3}
\begin{manualtheorem}{\ref{theo_abigx}}[Restated]
	(Fault class smearing in ABIGX) The fault class smearing degree in ABIGX is lower than that in IG:
	\begin{align}
		 FCS_{ABIGX}< FCS_{IG}  \notag
	\end{align}
\end{manualtheorem}
\begin{proof}
	Given the explained fault sample $\mathbf{x}$ of fault type $y$, the AFR-reconstructed sample on the linear model is optimized with the classification SPE objective (denote optimal weight $\mathbf{W}^*$ by $\mathbf{W}$ ): 
\begin{align}
	\mathbf{x}'_{AFR} &= \arg\min_{\mathbf{x}'} \Vert \mathbf{W}\mathbf{x}' - \mathbb{E}_{\mathbf{x}\sim\mathbf{X}_{0}}[\mathbf{W}\mathbf{x}] \Vert_2^2 \notag \\
&=\arg\min_{\mathbf{x}'} \Vert \mathbf{W}(\mathbf{x}' - \overline{\mathbf{x}_{0}} ) \Vert_2^2 
\end{align}
where $\overline{\mathbf{x}_{0}}=\overrightarrow{0} $ is the mean vector of normality samples. We simply perform a one-step gradient solution for AFR:
\begin{align}
	\label{eq_fcsafr}
	\mathbf{x} - \mathbf{x}'_{AFR} & = \eta \nabla_{x'=x} \Vert \mathbf{W}(\mathbf{x}' - \overline{\mathbf{x}_{0}} ) \Vert_2^2 \notag \\
	& = 2\eta {\mathbf{W}}^T\mathbf{W}\mathbf{x}
\end{align}
\par Denote $\mathbf{x}'_{AFR}$ by $\mathbf{x}'=[x'_1,\cdots,x'_M]$, the expectation of ABIGX variable contributions are:
\begin{align}
	\label{eq_fcsexcp}
	\mathbb{E}[ ABIGX_i]& =\begin{cases}
		\mathbb{E}_{x_i\sim  \mathcal{N}(f,\sigma^2)} \vert \int_{x=x'_{i}}^{x=x_y} w_1 \, dx \vert & i=y\\
		\mathbb{E}_{x_i\sim  \mathcal{N}(0,\sigma^2)} \vert \int_{x=x'_{i}}^{x=x_y} w_2 \, dx \vert & i\neq y
	\end{cases}
\end{align}
\par Omit the consistent scalar in eq. \ref{eq_fcsafr}, Eq. \ref{eq_fcsexcp} can be written as:
\begin{align}
	\mathbb{E}_{\mathbf{x} \sim \mathbf{X}_y} [ ABIGX_i]& =\begin{cases}
		\vert \mathbb{E}\ row_i({\mathbf{W}}^T\mathbf{W}\mathbf{x}) w_1 \vert & i=y\\
		\vert \mathbb{E}\ row_i({\mathbf{W}}^T\mathbf{W}\mathbf{x}) w_2 \vert & i\neq y
	\end{cases}
\end{align}
\par Based on Theorem \ref{theo_optweight}, the Hessian matrix ${\mathbf{W}}^T\mathbf{W}\in \mathbb{R}^{M\times M}$ is:
\begin{align}
	{\mathbf{W}}^T\mathbf{W}_{i,j} = \begin{cases}
		\frac{1}{N^2} + \frac{N^2+N-1}{(2N-1)^2}, & i=j\leq N \\
		\frac{1}{N^2} - \frac{N+2}{(2N-1)^2},  & i\neq j\leq N \\
		0, &otherwise
	\end{cases}
\end{align}
\par Let $w_3=\frac{1}{N^2} + \frac{N^2+N-1}{(2N-1)^2}$ and $w_4=\frac{1}{N^2} - \frac{N+2}{(2N-1)^2}$, we have the expected ABIGX variable contributions $\mathbb{E}_{\mathbf{x} \sim \mathbf{X}_y} [ABIGX_i]$:
\begin{align}
\begin{cases}
		\mathbb{E}_{x'\sim  \vert \mathcal{N}(0,\sigma^2) \vert}\vert (fw_3 + (N-1)x'w_4) w_1 \vert & i=y\\
		\mathbb{E}_{x'\sim \vert \mathcal{N}(0,\sigma^2) \vert} \vert (fw_4 + x'w_3 + (N-2)x'w_4) w_2 \vert & i\neq y
	\end{cases}
\end{align}
\par Then we obtain the FCS degree of ABIGX:
\begin{align}
	FCS_{ABIGX}
	&= \frac{N-1}{N} \frac{fw_4 + x'w_3 + (N-2)x'w_4}{fw_3 + (N-1)x'w_4}
\end{align}
\par We can get the sufficient and necessary condition for $FCS_{ABIGX} < FCS_{IG}$:
\begin{align}
	\frac{fw_4 + x'w_3 + (N-2)x'w_4}{fw_3 + (N-1)x'w_4} &<\frac{x'}{f} \notag\\
	f^2w_4 + (N-2)f x'w_4 &<(N-1){x'}^2w_4
\end{align}
\par Since $w_4<0$, the sufficient condition for $FCS_{ABIGX} < FCS_{IG}$ is:
\begin{align}
	f^2 + (N-2)f x'& >(N-1){x'}^2 \notag \\
	f^2& >{x'}^2
	\label{eq_suf1}
\end{align}
\par Since $f>x'$, we can prove that $FCS_{ABIGX} < FCS_{IG}$ always holds.
\end{proof}

\end{document}